\documentclass[11pt]{article}
\usepackage{pkg}
\usepackage{wrapfig}

\usepackage{hyperref}
\usepackage{color}
\usepackage[dvipsnames]{xcolor}
\usepackage{latexsym}
\usepackage{dsfont}
\usepackage{appendix}
\usepackage{amssymb}
\usepackage{subcaption}
\usepackage[margin=.7in]{geometry}
\usepackage{float}
\usepackage{algorithm,algorithmic}
\usepackage[sort&compress,numbers]{natbib}
\usepackage{amsmath, amsfonts,amssymb,euscript,array,amsfonts,mathrsfs}
\usepackage{enumerate}
\usepackage[dvipsnames]{xcolor}
\usepackage{amsthm}
\newtheorem{Theorem}{Theorem}[section]

\newtheorem{Proposition}[Theorem]{Proposition}
\newtheorem{Assumption}[Theorem]{Assumption}
\newtheorem{Lemma}[Theorem]{Lemma}

\newtheorem{Example}[Theorem]{Example}
\usepackage[T1]{fontenc}
\usepackage[utf8]{inputenc}
\usepackage{hhline}
\usepackage{graphicx}
\usepackage{verbatim}
\usepackage{eurosym}
\usepackage{dsfont}
\usepackage{mathtools}
\allowdisplaybreaks

\DeclareMathOperator{\Tr}{Tr}
\DeclareMathOperator{\E}{\mathbb{E}}

\DeclareMathOperator{\R}{\mathbb{R}}

\DeclareMathOperator*{\argmin}{arg\,min}

\makeatletter
\@addtoreset{equation}{section}

\newcommand*{\rom}[1]{\expandafter\@slowromancap\romannumeral #1@}
\makeatother
\newcommand{\vertiii}[1]{{\left\vert\kern-0.25ex\left\vert\kern-0.25ex\left\vert #1 
    \right\vert\kern-0.25ex\right\vert\kern-0.25ex\right\vert}}

\usepackage{soul}

\hypersetup{
    breaklinks=true,   %
    colorlinks=true,   %
    pdfusetitle=true,  %
 }

\begin{document}
\title{Policy gradient converges to the globally optimal policy for nearly linear-quadratic regulators}

\author{Yinbin Han, Meisam Razaviyayn, Renyuan Xu 
}
\author{Yinbin Han
\thanks{Daniel J. Epstein Department of Industrial and Systems Engineering, University of Southern California, Los Angeles, CA 90089, USA. \textbf{Email:} \{yinbinha, razaviya,  renyuanx\}@usc.edu}
\and
Meisam Razaviyayn \footnotemark[1]
\and
Renyuan Xu 
  \footnotemark[1]
}
\date{\today}
\maketitle

\begin{abstract}%
    Nonlinear control systems with partial information to the decision maker are prevalent in a variety of applications. As a step toward studying such nonlinear systems, this work explores reinforcement learning methods for finding the optimal policy in the {\it nearly linear-quadratic regulator} systems. In particular, we consider a dynamic system that combines linear and nonlinear components, and is governed by a policy with the same structure. Assuming that the nonlinear component comprises kernels with small Lipschitz coefficients, we characterize the optimization landscape of the cost function. Although the cost function is nonconvex in general, we establish the local strong convexity and smoothness in the vicinity of the global optimizer. Additionally, we propose an initialization mechanism to leverage these properties. Building on the developments, we design a policy gradient algorithm that is guaranteed to converge to the globally optimal policy with a linear rate. 
    \end{abstract}

\section{Introduction}

Reinforcement learning (RL) is one of the three classical machine learning paradigms, alongside supervised  and unsupervised learning. RL is learning via trial and error, through interactions with an environment and possibly with other agents. In RL, an agent takes actions and receives reinforcement signals in terms of numerical rewards encoding the outcome of the chosen action. %
In order to maximize the accumulated reward over time, the agent learns to select actions based
on past experiences (exploitation) and by making new choices (exploration).
In recent years, we have witnessed successful development of RL systems in various applications, including robotics control \cite{levine2016end, lillicrap2015continuous}, AlphaGo and Atari games \cite{mnih2013playing, 44806}, autonomous driving \cite{li2019reinforcement}, and stock trading \cite{deng2016deep}.  
Despite its practical success,  theoretical understanding of RL  is still  limited and at its primitive stage. 

\vspace{0.1cm}

To establish a better foundation of RL, there has been a surge of theoretical works in recent years on the Linear Quadratic Regulator (LQR) problem. This problem is a special class of control problems with linear dynamics and quadratic cost functions \cite{bu2019lqr,fazel2018global,hambly2021policy,malik2019derivative,mohammadi2019global,szpruch2021exploration,szpruch2022optimal,wang2020reinforcement}. 
In the seminal work of~\cite{fazel2018global}, the authors studied an LQR problem with deterministic dynamics over an infinite horizon. They proved that  the simple  policy gradient method  converges to the globally optimal solution with a linear  rate (despite nonconvexity of the objective). Their key idea is to utilize the Riccati equation (an {algebraic-equation} characterization that only works for LQR problems) and show that the cost function enjoys a ``gradient dominance'' property. 
This result has been extended to other settings such as linear dynamics with additive or multiplicative Gaussian noise, finite-time horizon, and modifications of the vanilla policy-gradient method in follow-up works \cite{bu2019lqr,hambly2021policy,malik2019derivative,mohammadi2019global}. Other aspects in the learning of LQR, such as the trade-off between exploration and exploitation,  have also been studied recently \cite{szpruch2021exploration,szpruch2022optimal,wang2020reinforcement}.

\vspace{0.1cm}

Despite the desirable theoretical properties of LQR, this setting is limited in practice due to the {\it nonlinear nature} of many real-world dynamic systems. From a technical perspective, it is unclear how much we can go beyond    the linear setting and still maintain the desirable properties of LQR. 
Our preliminary attempt in this direction is to study learning-based methods for linear systems perturbed by  some nonlinear kernel functions of small magnitude. Such systems are denoted as {\it nearly linear-quadratic systems} throughout the paper. The motivations for considering this setting are twofold: (1) Many nonlinear systems can be approximated by an LQR with a small nonlinear correction term via local expansions. %
(2) Analyzing the nearly linear-quadratic system provides a natural perspective to evaluate the stability of LQR systems. This could further address the question of how {\it robust} LQR framework is with respect to {\it  model mis-specifications} and, more broadly, how {\it reliable} the nearly linear-quadratic systems (including LQR problems as a special case) are.

\paragraph{Our Contributions.} 
We first study the optimization landscape of a special class of nonlinear control systems and  propose a policy-gradient-based algorithm to find the optimal policy. Specifically, we consider the nonlinear dynamics consisting of both linear and nonlinear parts. The nonlinear part is modeled by a linear combination of differentiable kernels with small Lipschitz coefficients. The kernel basis is known to the agent but the coefficients are not available to the agent. Additionally, we allow agents to apply nonlinear control policies in the form of the sum of a linear part and a nonlinear part where the nonlinear part lies in the same span of the kernel basis for the dynamics.
Our analysis shows that the cost function is {\it locally strongly convex} in a small neighborhood containing both a carefully chosen initial policy and the {\it globally optimal solution}. {Particularly, a least-squares regression method is proposed to obtain this desirable initial policy when model parameters are unknown.} With these results in hand, a zeroth-order policy-gradient method is proposed with guaranteed convergence to the globally optimal solution
with a linear rate.

\paragraph{Related Work.}
Our work is related to  three categories of prior work:

First, our framework and analysis tools are closely related to learning-based methods for the LQR problem and its variants. This includes policy gradient methods in \cite{bu2019lqr,fazel2018global, gravell2019learning,hambly2021policy,jansch2020convergence,malik2019derivative,mohammadi2019global, Zhang2021PolicyOF} and actor-critic methods in \cite{jin2020analysis,yang2019provably,zhang2021provably}. All these works focus on linear systems and linear policies, and show the property of global convergence.  In contrast, we step into the nonlinear world by examining the policy gradient method for nonlinear systems that are ``near-linear''  in a certain sense.

Second, our work lies within  the literature on nonlinear control systems. The work \cite{sastry2013nonlinear} provides a comprehensive review on this topic and \cite{Umlauft2017FeedbackLU,Umlauft2020FeedbackLB,westenbroek2020feedback,zeng2019adaptive} offer recent advances such as feedback linearization and neural network approximations.  \textit{Although largely inspired by \cite{Qu2020CombiningMA}, our work differs from it.} In \cite{Qu2020CombiningMA}, the dynamics consist of a linear component and a small, unknown nonlinear component. However, the authors only consider linear policies, whereas our framework allows for exploration of nonlinear control policies, which is more general and potentially leads to a better solution. Furthermore, in \cite{Qu2020CombiningMA}, the model parameters for the linear component are assumed to be known, which precludes the development of RL algorithms in a more general setting where the agent does not fully know the environment. %
In contrast, we model the linear component as unknown, and represent the nonlinear component as a linear combination of known kernel basis with unknown coefficients. We also propose a least-squares regression method to recover the system dynamics with no error (under high probability). This approach allows us to develop sample-based analysis in an unknown environment setting, which is not possible with the above-mentioned assumption in \cite{Qu2020CombiningMA}. To the best of our knowledge, this is the first theoretical study that demonstrates the global convergence of an RL method for a system with both  nonlinear dynamics (with continuous state and action spaces) and nonlinear control policies in the learning context. 

Finally, our work is related to the line of work on policy gradient methods. In addition to LQR, policy gradient methods have been applied to learn Markov decision processes (MDPs) with finite state and action spaces.
The recent developments that provide global convergence guarantees for policy gradient methods and their variants can be found in \cite{bhandari2021linear, Agarwal2021OnTT,Cen2021FastGC,Ding2020NaturalPG,Fu2021SingleTimescaleAP,Liu2019NeuralPR,Liu2020AnIA,Wang2020NeuralPG,xiao2022convergence,Xu2021DoublyRO,zhang2020variational,zhang2021provably,fatkhullin2023stochastic}.

\paragraph{Notation.} In this work, $\norm{{}\cdot{}}$ is always the 2-norm of vectors and matrices, and $\norm{{}\cdot{}}_F$ is the Frobenius norm of matrices. Additionally, $y_1 \lesssim y_2, y_1 \asymp y_2 $ and $ y_1 \gtrsim y_2$ mean $y_1 \leq cy_2, y_1 = cy_2$ and $y_1 \geq cy_2$ for some absolute constant $c > 0$, respectively. 
 
\section{Problem Setup} \label{sec: setup}
We consider a dynamical system with the state variable $x_t\in \mathbb{R}^n$ and the control variable $u_t\in \mathbb{R}^p$:
\begin{eqnarray}\label{eq:dynamics}
x_{t+1} = A x_t  +C\phi(x_t) +B u_t,
\end{eqnarray}
where $A\in \mathbb{R}^{n\times n}$, $C\in \mathbb{R}^{n\times d}$, $B\in \mathbb{R}^{n\times p}$, and a kernel basis $\phi(x) = (\phi_1(x),\cdots,\phi_d(x))^{\top}$ with  $\phi_i(x):\mathbb{R}^n\rightarrow \mathbb{R}$ $(i=1,2,\cdots,d)$. Here, $\phi(x)$ satisfies certain Lipschitz continuity conditions (specified later in Assumption \ref{ass:feature}). The system in \eqref{eq:dynamics} is the summation of a linear part and a ``small'' nonlinear part. The nonlinear part is a (finite) linear combination of kernel basis. Essentially, the dynamics in \eqref{eq:dynamics} can be viewed as a nonlinear system that closely approximates a linear model. Additionally, \eqref{eq:dynamics} is more general than the linear systems considered in \cite{fazel2018global,hambly2021policy,malik2019derivative}, and therefore better represents the behaviors of a broader class of dynamic systems in practice. Despite its nonlinearity, we will show that \eqref{eq:dynamics} still enjoys desirable theoretical properties that are not present in fully nonlinear systems.

\vspace{0.1cm}

The admissible control set contains a class of  stationary Markovian policies that are linear combinations of the current state and kernels of the current state, i.e., 
\begin{align}\label{eq: policy}
    u_t = -K_1 x_t - K_2 \phi(x_t),
\end{align}
with $K_1\in\mathbb{R}^{p\times n}$ and $K_2\in \mathbb{R}^{p\times d}$. The form of the Markovian policies in \eqref{eq: policy} is motivated 
by the additive structure in the system dynamics \eqref{eq:dynamics}, with the same kernel $\phi$ involved. Additionally, we consider the following  domain $ \Omega $ (i.e., the admissible control set) for $ K= (K_1, K_2) $:
\begin{eqnarray}
\Omega = \curly{K:\,\, \norm{(A - BK_{1})^{t}} \leq c_{1}\rho_{1}^{t}, \,{\forall t \geq 1},\, \norm{C - BK_{2}} \leq c_{2}},\label{eq:omega}
\end{eqnarray}
for some $ c_{1} > 1, \rho_{1} \in (0, 1) $, and $ c_{2} > 1  $ (to be specified later). In general, characterizing the stabilizing region of a nonlinear system is challenging. Thus, we mirror the notions used in \cite{Qu2020CombiningMA} to consider the region $\Omega$ such that the control policy enjoys asymptotic stability. We will show that if the nonlinear part $\phi$ is ``small'', the controller in $ \Omega $ is asymptotically stable, i.e., $ \norm{x_{t}} \to 0 $ as $ t \to \infty $.
Furthermore, we consider the  quadratic cost function $\mathcal{C}: \R^{p \times (n+d)}  \to \R $ with $K = (K_1,K_2)$:
\begin{eqnarray}\label{eq:cost}
{ \mcal{C}}(K) = \mathbb{E}_{x_0 \sim \mcal{D}}\left[ \sum_{t=0}^{\infty}x_t^{\top} Q x_t + u_t^{\top}R u_t\right],
\end{eqnarray}
where the expectation is taken with respect to $x_0$ (drawn from {\it an unknown} distribution $\mathcal{D}$). The state trajectory $\curly{x_t}_{t=0}^{\infty}$ is generated via the control policy $K$ defined in \eqref{eq: policy}.  Here, $Q$ and $R$ are symmetric positive-definite matrices. Thus, the running cost $c_t = x_t^{\top} Q x_t + u_t^{\top}R u_t $ is quadratic in both the state and control variables. The agent and the environment interact in the following way:~At the beginning of each time step $t = 0, 1, 2,\dots$, the agent receives the state $x_t$ that encodes the full information of the environment and chooses a control $u_t$. At the end of this time step, the agent receives an instantaneous cost $c_t$ and a new state $x_{t+1}$ as a consequence of the control input. {\color{black}The agent has the option to restart the system at any time step. This can be achieved by, for example, accessing to a generative model that can generate sample trajectories.} The objective is to find the optimal  policy $K$ that minimizes the cost function $\mcal{C}(K)$ when the model parameters ($A$, $B$ and $C$) are unknown.

\begin{algorithm}[t!]
    \caption{Policy Gradient Estimation}\label{alg:policy gradient estimation}
    \begin{algorithmic}[1] 
    \STATE {\bf {Input:}}~Policy $K = (K_1, K_2)$, number of trajectories $J$, smoothing parameter $r$, and episode length $T$. 
    \FOR{$j = 1, 2, \dots, J$}
    \STATE Sample a policy $\widehat{K}^j = K + U^j$, where $U^j$ is drawn uniformly at random over matrices of size $ p\times (n+d) $ whose Frobenius norm is $r$.\label{line: perturb}
    \STATE Sample $x_0^j \sim \mathcal{D}$.
    \FOR{$t = 0, 1, \dots, T$}
    \STATE Set $u_t = -\widehat{K}_1^j x_t - \widehat{K}_2^j \phi(x_t)$.
    \STATE Receive the cost $c_t$ and the next state $x_{t+1}$ from the system.
    \ENDFOR
    \STATE Calculate the estimated cost $\widehat{\mcal{C}}_j = \sum_{t = 0}^{T}c_t $.
    \ENDFOR
    \RETURN  $\widehat{\nabla \mcal{C}(K)}  = \frac{1}{J}\sum_{j = 0}^{J}\frac{\widehat{D}}{r^2}\widehat{\mcal{C}}_j U^j $, where $\widehat{D} =p(n+d) $. \label{line: gradient estimation}
    \end{algorithmic}
  \end{algorithm}

\section{Proposed Algorithm}\label{sec: algorithm}
{\color{black}The main difficulties of the control problem \eqref{eq:dynamics}--\eqref{eq:cost} are the unknown dynamics \eqref{eq:dynamics} and the nonconvexity of the objective \eqref{eq:cost}, especially in high-dimensional scenarios.}
Given that any admissible control policy defined in \eqref{eq: policy} can be fully characterized by a policy parameter $ K  $ in $\Omega$, we leverage policy gradient methods in \cite{sutton1999policy} to find the optimal policy $K^{*}$. When all the model parameters are known to the decision maker (referred to as the {\it model-based} case), policy gradient methods iteratively update the (current) policy $K$ by utilizing the gradient information $\nabla \mathcal{C}(K)$. When the model parameters are unknown (referred to as the {\it model-free} case), the gradient term $\nabla\mathcal{C}(K)$ can be replaced by an estimate $\widehat{\nabla\mathcal{C}(K)}$ to perform an approximate gradient descent step. {\color{black}In both cases, the initial distribution $\mathcal{D}$ is unknown while samples from $\mcal{D}$ are available to the agent. }

We now present our policy gradient algorithm to learn the optimal control for problem \eqref{eq:dynamics}-\eqref{eq:cost}. 
\paragraph{Zeroth-order Optimization Method.} Using a zeroth-order optimization framework,
Algorithm~\ref{alg:policy gradient estimation} provides  an estimate $\widehat{\nabla \mcal{C}(K)}$ for the policy gradient $\nabla \mcal{C}(K)$. This estimate will later be used in the following policy gradient update rule:
\begin{align}\label{eq:algo}
    K^{(m+1)} = K^{(m)} - \eta \widehat{\nabla \mcal{C}(K^{(m)})}, \quad K^{(0)} = K^{\rm lin},
\end{align}
where   $K^{\rm lin} = (K_{1}^{\rm lin}, K_{2}^{\rm lin})$ is the initial policy, which will be chosen carefully to obtain an efficient convergence to the global optimum, see the next part, {\it Efficient
Initialization}.

{Our zeroth-order estimate (line \ref{line: gradient estimation} in Algorithm \ref{alg:policy gradient estimation}) approximates the gradient of the function $\mathcal{C}$ by using the function values.  Note that $ \E\bracket{U} = 0 $ as $U$ is uniformly distributed over a sphere of a ball with radius $r$ (Frobenius norm). The first-order Taylor expansion of $ \mcal{C} $ leads to
\begin{eqnarray}
 \label{eq:gradient_estimate}
\E\bracket{\mathcal{C}(K+U)U} \approx \E\bracket{ \parenthesis{\nabla \mathcal{C}(K)^{\top} U}U} = \nabla\mathcal{C}(K)r^2/\widehat{D},
\end{eqnarray}
where $K \in \R^{\widehat{D}}$ with $\widehat{D} = p(n+d)$ and $U$ is uniformly distributed over a sphere of a ball with radius $r$ (Frobenius norm). Hence, to compute the estimate $\widehat{\nabla \mathcal{C}(K)}$ and to approximate the expectation in \eqref{eq:gradient_estimate} under an input policy $K$, Algorithm \ref{alg:policy gradient estimation} collects $J$ sample trajectories. Each trajectory follows a perturbed policy $\widehat{K} = K + U$ (line \ref{line: perturb}). Finally, the gradient estimate can be obtained by averaging over the sample trajectories $\frac{\widehat{D}}{r^2}\mcal{C}(K+U) U$. Lemma \ref{lemma: estimation C(K)} in Section \ref{sec: proofs} will show that if $J \gtrsim \frac{\widehat{D}^2}{e_{\rm grad}^2}\log \frac{4\widehat{D}}{\nu}$, it holds with probability at least $1 - \nu$ that 
\begin{align*}
    \norm{\widehat{\nabla \mcal{C}}(K) - \nabla \mcal{C}(K)}_{F} \leq e_{\text{grad}}.
\end{align*}
}

\paragraph{Efficient Initialization.} As recognized in \cite{Qu2020CombiningMA}, the cost function $\mathcal{C}(K)$ may have many spurious local minima due to its nonconvex nature. Consequently, a policy gradient method with an arbitrary initialization may fail to converge to the global minimizer. Interestingly, we present a design for an initialization, denoted by $K^{\rm lin} = (K_{1}^{\rm lin}, K_{2}^{\rm lin})$, which ensures it lies within the  {\it basin of attraction} of the globally optimal solution. 
Specifically, we choose $ K_{1}^{\rm lin} $  to be the optimal control policy of the following linear-quadratic problem:
\begin{align}\label{eq: lqr}
\min_{K_{1}} & \quad \E_{x_{0} \sim \mathcal{D}}\bracket{\sum_{t = 0}^{\infty}x_{t}^{\top}Qx_{t} + u_{t}^{\top}Ru_{t}}, \nonumber\\ 
\text{subject to} & \quad x_{t+1} = Ax_{t} + Bu_{t}, u_{t} = -K_{1}x_{t}.
\end{align}
{The LQR problem defined in \eqref{eq: lqr} is a special instance of  the problem described in \eqref{eq:dynamics}--\eqref{eq:cost}, obtained by setting $C = 0$ and $K_2 = 0$. The intuition behind this LQR problem is as follows. When the nonlinear term $\phi(x)$ is ``small'' (see Assumption \ref{ass:feature} for a mathematical description), the optimal policy $K_1^{\rm lin}$ for the LQR problem is anticipated to be close to the optimal controller $K_1^*$ for the nonlinear problem \eqref{eq:dynamics}, leading to a potentially useful initialization.} Coming back to the LQR problem, it is well-established in the control literature \citep{anderson2007optimal,bertsekas2012dynamic}    that the policy $ K_{1}^{\rm lin} $ is unique when the pair $ (A, B) $ is controllable. To define this unique policy, let the positive definite matrix $P$ be the unique solution to the Algebraic Riccati Equation (ARE),
\begin{align}\label{eq:ARE}
    P = A^{\top}PA + Q - A^{\top}PB(R + B^{\top}PB)^{-1}B^{\top}PA.
\end{align}

{\color{black}\textbf{Model-based setting.} When all the model parameters, $Q, R, A, B, $ and $C$, are {\it known}}, the optimal controller for the problem \eqref{eq: lqr} is given as:
\begin{eqnarray}\label{eq:linear_initialization}
    K_1^{\rm lin} = (R + B^{\top}PB)^{-1}B^{\top}PA.
\end{eqnarray}
Further, we set 
\begin{align}\label{eq:nonlinear_initialization}
    K_2^{\rm lin} = (R + B^{\top}PB)^{-1}B^{\top}PC.
\end{align}
{\color{black}We will show that the initial policy $K^{\rm lin}$ defined above is close to the optimal solution $K^*$ when the nonlinear term $\phi(x)$ is ``small''.}

\begin{algorithm}[bt!]
    \caption{Estimation of the System Dynamics' Parameters with Independent Data}\label{alg:init estimation}
    \begin{algorithmic}[1] 
    \STATE {\bf {Input:}}~Number of samples $ N $.
    \FOR{$i = 1, 2, \dots, N$} \label{line: sampling starts}
    \STATE Sample $ x_{0}^{(i)} \overset{\text{i.i.d.}}{\sim} \mcal{D} $, $ u_{0}^{(i)} \overset{\text{i.i.d.}}{\sim} \begin{cases}
        \mcal{N}(0, I_{p}), \ & \text{w.p.} \ 1/2 \\ 0, \ &\text{w.p.} \ 1/2 \label{line: dist u_0}
    \end{cases}$ and observe $ x_{1}^{(i)} = Ax_{0}^{(i)} + Bu_{0}^{(i)} + C\phi(x_{0}^{(i)}) $. \label{line: iteration i}
    \ENDFOR \label{line: sampling ends}
    \RETURN  $(\widehat{A}, \widehat{B}, \widehat{C})$ by solving \eqref{eqn: least squares}.
    \end{algorithmic}
\end{algorithm}

{\color{black}\textbf{Model-free setting.} When the model parameters $ A, B $ and $ C $ are unknown, one key challenge lies in finding an appropriate initialization $ K^{\rm lin} $. We address this issue by utilizing the least-squares estimators of the parameters $ A, B $ and $ C $.}
This estimation process is described in Algorithm \ref{alg:init estimation}. In iteration $i$  of Algorithm \ref{alg:init estimation} (see line \ref{line: iteration i}), the system starts at a state $ x_{0}^{(i)} \sim \mathcal{D} $, and the dynamics evolve to the next state $ x_{1}^{(i)} $ under the control $u_0^{(i)}$. {\color{black}Here, we randomly draw the control $ u_{0}^{(i)} $ from a certain distribution (line \ref{line: dist u_0}) to guarantee that the parameters $ A $, $ B $ and $ C $ can be recovered with high probability.}  This operation is repeated for $N$ times, resulting in a dataset of the form $ \curly{\parenthesis{x_{0}^{(i)}, u_{0}^{(i)}, x_{1}^{(i)}}: 1 \leq i \leq N} $. Based on this dataset,  we estimate the system parameters through the following least-squares minimization procedure:
\begin{align}\label{eqn: least squares}
(\widehat{A}, \widehat{B}, \widehat{C}) = \argmin_{\tilde{A}, \tilde{B}, \tilde{C}}\frac{1}{2}\sum_{i = 1}^{N}\norm{\tilde{A}x_{0}^{(i)}+\tilde{B}u_{0}^{(i)} + \tilde{C}\phi(x_{0}^{(i)}) - x_{1}^{(i)}}^{2}.
\end{align}
 When the cost parameters  $Q$ and  $R$ are known \cite{dean2020sample}, we can use the estimated values $ \widehat{A}, \widehat{B} $ and $ \widehat{C} $ to initialize $ K_{1}^{\rm lin} $ and $ K_{2}^{\rm lin} $ in \eqref{eq:linear_initialization} and \eqref{eq:nonlinear_initialization}, repectively. Precisely, this is achieved by setting:
\begin{align*}
    \widehat{K}_1^{\rm lin} & = (R + \widehat{B}^{\top}\widehat{P}\widehat{B})^{-1}\widehat{B}^{\top}\widehat{P}\widehat{A} \quad \text{and} \quad
    \widehat{K}_2^{\rm lin}  = (R + \widehat{B}^{\top}\widehat{P}\widehat{B})^{-1}\widehat{B}^{\top}\widehat{P}\widehat{C} ,
\end{align*}
where $\widehat{P}$ is the solution to the ARE, $\widehat{P} = \widehat{A}^{\top}\widehat{P}\widehat{A} + Q - \widehat{A}^{\top}\widehat{P}\widehat{B}(R + \widehat{B}^{\top}\widehat{P}\widehat{B})^{-1}\widehat{B}^{\top}\widehat{P}\widehat{A}.$

In the next section, we will show that: 
\begin{itemize}
    \item With high probability, the least-squares regression \eqref{eqn: least squares} fully recovers the exact parameters, $A, B $ and $ C$, with no estimation error, i.e., $ (A, B, C) = (\widehat{A},\widehat{B}, \widehat{C}) $.
    \item The optimal solution to the nonlinear control problem
    \eqref{eq:dynamics}--\eqref{eq:cost} lies within a small neighborhood of the initial policy $ K^{\rm lin}$. 
    \item The cost function \eqref{eq:cost} is strongly convex and smooth in a neighborhood containing both the initial policy $ K^{\rm lin}$ and the globally optimal policy $ K^{*} $.
\end{itemize}
{The first result implies that the least-squares regression provides  the exact initial policy $\widehat{K}_1^{\rm lin} = {K}_1^{\rm lin}$ and $\widehat{K}_2^{\rm lin} = {K}_2^{\rm lin}$ when the model parameters are unknown.}
The last two facts
will establish the convergence of Algorithm \ref{alg:policy gradient estimation} to the global optimum.

\section{Main Results}\label{sec: main results}

In this section, we present our main theoretical results. We first prove the recovery property of Algorithm \ref{alg:init estimation} introduced in Section \ref{sec: algorithm}. Next, we proceed to characterize the optimization landscape of the cost function. In particular, we show the local strong convexity of the cost function around its global minimum. Furthermore, we prove that the globally optimal solution is close to our carefully chosen initialization. Finally, we establish the  convergence  of Algorithm~\ref{alg:policy gradient estimation}. Before stating our main results, we make the following assumptions for problem \eqref{eq:dynamics}--\eqref{eq:cost}. %

\begin{Assumption}\label{ass:feature}
    We assume that $ \phi: \R^{n} \to \R^{d} $ is differentiable, $ \phi(0) = 0 $ and $ \norm{\phi(x) - \phi(x')} \leq \ell \norm{x - x'} $ for any $ x, x' \in \R^{n} $ with $\ell>0$. Moreover, we assume that $\norm{\nabla \phi(x) - \nabla \phi(x')} \leq \ell'\norm{x - x'} $ for any $ x, x' \in \R^{n} $ with $ \ell' > 0 $. 
\end{Assumption}

Assumption \ref{ass:feature} states that the kernel function $\phi$ is $\ell$-Lipschitz and $\ell'$-gradient-Lipschitz. The examples of $ \phi $ are not restrictive. Let us provide two kernel basis examples that satisfy Assumption \ref{ass:feature}. 
The first example is {$\phi_i(x) = \alpha_{i}\sin x$ with $\alpha_i \geq 0$}, for which Assumption \ref{ass:feature} holds automatically. {\color{black}The second example is introduced below.}

\begin{Example}\label{prop: kernel}
Let $ N_{i} \geq 0, i = 1, \dots, d $, be non-negative integers. For fixed $ w_{j}^{i} \in \curly{-1, 1}^{n}, j = 1, \dots, N_{i} $, define $ \phi_{i}(x) = \alpha_{i}\prod_{j = 1}^{N_{i}}x^{\top}w_{j}^{i} $ with $ \alpha_{i} \geq 0 $. Then if $ \norm{x} \leq M_{0} $, the kernel basis $ \phi(x) = (\phi_{1}(x), \dots, \phi_{d}(x))^{\top} $ satisfies Assumption \ref{ass:feature} with $ \ell = \parenthesis{\sum_{i = 1}^{d}n^{2}\alpha_{i}^{2}N_{i}^{2}M_{0}^{2(N_{i}-1)}}^{1/2} $ and $ \ell' = \parenthesis{\sum_{i = 1}^{d}n^{3}\alpha_{i}^{2}N_{i}^{2}(N_{i} - 1)^{2}M_{0}^{2(N_{i}-2)}}^{1/2} $.
\end{Example}
Note that the class of kernel basis  in Example \ref{prop: kernel} {is used in  kernel-based methods for supervised learning, unsupervised learning, nonparametric regression, and offline RL \cite{kar2012random, steinwart2008support, farahmand2016regularized,hardle1990applied, scholkopf1997kernel}.}

\begin{proof}[Proof of Example \ref{prop: kernel}]
   It is straightforward to check that $ \phi(0) = 0 $. To prove the Lipschitz and gradient-Lipschitz properties, it suffices to show that the properties hold for each $ \phi_{i}(x) $. Indeed, if $ \phi_{i}(x) $ is $ \ell_{i} $-Lipschitz and $ \ell_{i}'$-gradient-Lipschitz, then we have
   \begin{align*}
   \norm{\phi(x) - \phi(y)}^{2} = \sum_{i = 1}^{d}\norm{\phi_{i}(x) - \phi_{i}(y)}^{2} \leq \parenthesis{\sum_{i = 1}^{d}\ell_{i}^{2}}\norm{x - y}^{2}.
   \end{align*}
   It follows that $ \phi(x) $ is $ \parenthesis{\sum_{i = 1}^{d}\ell_{i}^{2}}^{1/2} $-Lipschitz. To see $ \phi $ is also gradient-Lipschitz, note that
   \begin{align*}
   \norm{\nabla\phi(x) - \nabla\phi(y)}^{2}_{2} & \leq \norm{\nabla\phi(x) - \nabla\phi(y)}_{F}^{2} = \sum_{i = 1}^{d}\norm{\nabla\phi_{i}(x) - \nabla\phi_{i}(y)}^{2} \leq \parenthesis{\sum_{i = 1}^{d}(\ell_{i}')^{2}}\norm{x - y}^{2}.
   \end{align*}
   It follows that $ \nabla \phi(x) $ is $ \parenthesis{\sum_{i = 1}^{d}(\ell_{i}')^{2}}^{1/2} $-Lipschitz. 

   The rest of the proof is to show the Lipschitz continuity of $ \phi_{i} $ and $ \nabla \phi_{i} $. We first compute the Lipschitz constant of $ \phi_{i} $. Since $ \nabla \phi_{i}(x) = \alpha_{i}\sum_{j = 1}^{N_{i}}(\prod_{k \neq j}x^{\top}w_{k}^{i})w_{j}^{i} $, we have
   \begin{align*}
   \norm{\nabla \phi_{i}(x)} \leq \alpha_{i}\sum_{j = 1}^{N_{i}}\norm{w_{j}^{i}}\prod_{k \neq j}\norm{x}\norm{w_{k}^{i}} \leq n\alpha_{i}N_{i}M_{0}^{N_{i}-1}. 
   \end{align*}
   Also, we have $ \nabla^{2}\phi_{i}(x) = \alpha_{i}\sum_{j = 1}^{N_{i}}\nabla_{x}\parenthesis{\prod_{k \neq j}x^{\top}w_{k}^{i}}w_{j}^{i} = \alpha_{i}\sum_{j = 1}^{N_{i}}\sum_{k \neq j}\prod_{l \neq k, j}(x^{\top}w_{l}^{i}) w_{k}^{i}(w_{j}^{i})^{\top} $. Thus,
   \begin{align*}
   \norm{\nabla^{2}\phi_{i}(x)} \leq \alpha_{i}\sum_{j = 1}^{N_{i}}\sum_{k \neq j}\parenthesis{\prod_{l \neq k, j}\norm{x}\norm{w_{l}^{i}}}\norm{w_{k}^{i}}\norm{w_{j}^{i}} \leq n^{3/2}\alpha_{i}N_{i}(N_{i} - 1)M_{0}^{N_{i}-2}.
   \end{align*}
   Finally, we conclude the proof with $ \ell_{i} = n\alpha_{i}N_{i}M_{0}^{N_{i}-1} $ and $ \ell_{i}' =  n^{3/2}\alpha_{i}N_{i}(N_{i} - 1)M_{0}^{N_{i}-2} $.
\end{proof}

Next, we make the following standard assumption on the  matrices $Q$ and $R$ (see also~\cite{mania2019certainty}).

{\color{black}
\begin{Assumption}\label{ass: cost function}
    We assume that $ Q $ and $ R $ are positive definite matrices with $ \norm{Q}, \norm{R} \leq 1 $.
\end{Assumption}}

The first part of the assumption guarantees that the cost function has quadratic growth and therefore renders the problem well-defined \cite{Qu2020CombiningMA}. {\color{black}For convenience, we denote $\sigma \coloneqq \lambda_{\min} (R + B^{\top}QB) $, the smallest eigenvalue of the matrix.} The upper bound one  (on the norms of  $Q$ and $R$) in Assumption \ref{ass: cost function} is for ease of presentation and can be generalized to any arbitrary value by rescaling the cost function. Our subsequent assumption concerns the initial distribution of the state dynamics. 

\begin{Assumption}\label{ass:initial_distr}
  We assume that the initial distribution $ \mcal{D} $ is supported in a region with radius $ D_{0} $, i.e., $ \norm{x} \leq D_{0} $  for $ x \sim \mcal{D} $ with probability one. Also, we assume $ \E \left[\psi(x_{0})\psi(x_{0})^{\top}\right] \succeq \sigma_{x}I $ for some $ \sigma_{x} > 0 $, {where $\psi(x_0) = (x_0^{\top}, \phi(x_0)^{\top})^{\top}$}.
\end{Assumption}

 Assumption \ref{ass:initial_distr} requires the state initial distribution to be bounded. This assumption simplifies the proof in the subsequent sections, and can be relaxed by assuming an upper bound on the second and the third moments of the initial state \cite{Qu2020CombiningMA}. Also,   the covariance matrix $\E \bracket{\psi(x_{0})\psi(x_{0})^{\top}}$  is assumed to be bounded below by a positive constant matrix $\sigma_{x}I$. This   ``diverse covariate'' assumption ensures sufficient exploration (in all directions of the state space) even with a greedy algorithm. Finally, we lay out another regularity condition on the coefficient matrices $(A, B)$ and the initial policy $K^{\rm lin}$.
\begin{Assumption}\label{ass: initialization}
    The pair $ (A, B) $ is controllable. 
\end{Assumption}

The controllablity assumption on the pair $(A, B)$ is  standard in the literature \cite{bu2019lqr}. Assumption \ref{ass: initialization} implies that  the initial controller $ K^{\rm lin} = (K_{1}^{\rm lin}, K_{2}^{\rm lin}) $ defined in Section \ref{sec: algorithm} enjoys a stability property; that is, $ \norm{(A - BK_{1}^{\rm lin})^{t}} \leq c_{1}^{\rm lin}(\rho_{1}^{\rm lin})^{t} $ for all $t \geq 1$, and  $  \norm{C - BK_{2}^{\rm lin}} \leq c_{2}^{\rm lin} $ for some $ \rho_{1}^{\rm lin} \in (0, 1) $ and $ c^{\rm lin}_{1},  c^{\rm lin}_{2} > 0 $.

\vspace{0.1cm}

\subsection{Least-Squares Regression for Parameters Recovery}\label{sec:lsr}
In this subsection, we show that the least-squares regression in Algorithm \ref{alg:init estimation} exactly recovers all the parameters, $A, B$ and $C$, in the system dynamics. For ease of exposition, define $ \varphi(x, u) = [x^{\top}, u^{\top}, \phi(x)^{\top}]^{\top} \in \R^{n + p + d} $ and $ \Theta = [A, B, C]^{\top} \in \R^{(n + p + d) \times n} $. Then the system dynamics at time $ t = 1 $ can be written as
\begin{align*}
x_{1}^{\top} = \varphi(x_{0}, u_{0})^{\top}\Theta.
\end{align*}

In lines \ref{line: sampling starts}--\ref{line: sampling ends} of Algorithm \ref{alg:init estimation}, we collect $N$ samples $ \curly{\parenthesis{x_{0}^{(i)}, u_{0}^{(i)}, x_{1}^{(i)}}: 1 \leq i \leq N} $. By denoting
\begin{align*}
    X_{N}^{\top} = 
    \begin{bmatrix}
        x_{1}^{(1)}, & \cdots, & x_{1}^{(N)}
    \end{bmatrix} \quad \text{and} \quad
    \Phi_{N}^{\top} = 
    \begin{bmatrix}
        \varphi\parenthesis{x_{0}^{(1)}, u_{0}^{(1)}}, & \cdots, &  \varphi\parenthesis{x_{0}^{(N)}, u_{0}^{(N)}}
    \end{bmatrix},
\end{align*}
we have $$ X_{N} = \Phi_{N}\Theta. $$  If the matrix $ \Phi_N^{\top}\Phi_N $ is invertible, the least-squares estimator  can be written as
\begin{align}\label{eq:Theta-hat}
\widehat{\Theta} & = \parenthesis{\Phi_N^{\top}\Phi_N}^{-1}\Phi_N^{\top}X_{N}.
\end{align}
Combining the above two results, we conclude that $\widehat{\Theta} = \Theta$ {if $ \Phi_N^{\top}\Phi_N $ is invertible.}
Proposition \ref{prop: inv of design matrix} guarantees that $ \Phi_{N}^{\top}\Phi_{N} $ is invertible with high probability. {\color{black}Furthermore, the number of samples required by Algorithm \ref{alg:init estimation} is much less than solving independent linear equations.} Analog to the analysis in \cite{dean2020sample}, we utilize the structure of the system dynamics and leverage recent results in the non-asymptotic analysis of random matrices to establish Proposition \ref{prop: inv of design matrix}.

\begin{Proposition}\label{prop: inv of design matrix}
    {Assume Assumptions \ref{ass:feature} and \ref{ass:initial_distr} hold.} For any $ \nu \in (0, 1) $, $ \Phi_{N}^{\top}\Phi_{N} $ is invertible for all $ N \gtrsim n+p+d $ with probability at least $ 1 - \nu $.  In consequence, $\widehat{\Theta} = \Theta$ {with probability at least $ 1 - \nu $.}
\end{Proposition}

Proposition \ref{prop: inv of design matrix} implies that, {with high probability}, least-squares regression recovers all the model parameters, $A, B$ and $C$, with no estimation error. As a consequence, we conclude $\widehat{K}_1^{\rm lin} = {K}_1^{\rm lin}$ and $\widehat{K}_2^{\rm lin} = {K}_2^{\rm lin}$ with high probability. Note that in Proposition \ref{prop: inv of design matrix}, there are $ n(n+p+d)$ parameters to be estimated. Our results guarantee that  $ \mathcal{O}(n+p+d) $ samples of dimension $n$ are sufficient to recover the exact values of the parameters. This bound appears to be optimally dependent on the parameters $ n, p $, and $ d $. The proof of Proposition \ref{prop: inv of design matrix} relies on Lemmas \ref{lemma: sub-gaussian rows} and \ref{lemma: inv of sed mom matrix} which are deferred to Section \ref{sec: proof of matrix inv}.

\begin{proof}[Proof of Proposition \ref{prop: inv of design matrix}]
    By a slight abuse of notation, let $ \Sigma = \E\bracket{\varphi^{(i)}\parenthesis{\varphi^{(i)}}^{\top}} $ with $ \varphi^{(i)} = \varphi(x_{0}^{(i)}, u_{0}^{(i)}) $. With the choice of $ u_{0}^{(i)} $ in Algorithm \ref{alg:init estimation}, the matrix $ \Sigma $ is invertible by Lemma \ref{lemma: inv of sed mom matrix}. Let $ Y_{N} = \Phi_{N}\Sigma^{-1/2} $. The $ i $-th row of the matrix $ Y_{N} $ is
\begin{align*}
\parenthesis{Y_{N}^{(i)}}^{\top} = \parenthesis{\Phi_{N}^{(i)}}^{\top}\Sigma^{-1/2}.
\end{align*}
Since $ \norm{x_{0}^{(i)}} \leq D_{0} $ with probability one and $ \phi $ is $ \ell $-Lipschitz, the rows of $ Y_{N} $ are independent sub-Gaussian random vectors. Furthermore, note that
\begin{align*}
\E\bracket{\parenthesis{Y_{N}^{(i)}}\parenthesis{Y_{N}^{(i)}}^{\top}} & = \Sigma^{-1/2}\E\bracket{\parenthesis{\Phi_{N}^{(i)}}\parenthesis{\Phi_{N}^{(i)}}^{\top}}\Sigma^{-1/2} = I_{n+p+d}.
\end{align*}
By Lemma \ref{lemma: sub-gaussian rows}, for each $ \nu \in (0, 1) $, $ Y_{N}^{\top}Y_{N} $ is invertible with probability at least $ 1 - \nu $ for any $ N \geq N_{0} \coloneqq \parenthesis{d_1\sqrt{n+p+d} - \sqrt{\frac{1}{d_2}\log\frac{2}{\nu}} }^{2}$. As a consequence, we have 
\begin{align*}
\parenthesis{\Phi_{N}^{\top}\Phi_{N}}^{-1} = \Sigma^{-1/2}\parenthesis{Y_{N}^{\top}Y_{N}}^{-1}\Sigma^{-1/2}
\end{align*}
holds with probability at least $ 1 - \nu $.
\end{proof}

\subsection{Landscape and Convergence Analysis}

In this subsection, we study the convergence rate for the policy gradient method introduced in Section \ref{sec: algorithm}. Our first theorem characterizes the landscape of the cost function. It shows that the cost function is strongly convex and smooth in a region of the initialization $K^{\rm lin}$ when the Lipschitz constants $\ell$ and $\ell'$  are sufficiently small. Further, we prove the optimal controller $K^{*}$ is inside this neighborhood. {Denote $ \Gamma = \max\curly{\norm{A}, \norm{B}, \norm{C}, \norm{K^{\rm lin}}_{F}, 1} $.} Recall that the initial policy $K^{\rm lin} = (K^{\rm lin}_1, K^{\rm lin}_2)$ satisfies $ \norm{(A - BK_{1}^{\rm lin})^{t}} \leq c_{1}^{\rm lin}(\rho_{1}^{\rm lin})^{t} $ for all $t \geq 1$, and  $  \norm{C - BK_{2}^{\rm lin}} \leq c_{2}^{\rm lin} $ for some $ \rho_{1}^{\rm lin} \in (0, 1) $ and $ c^{\rm lin}_{1},  c^{\rm lin}_{2} > 0 $. Having these definitions in mind, let us formally state our main result:

\begin{Theorem}\label{thm: landscape}
    Assume Assumptions \ref{ass:feature} and \ref{ass: cost function}--\ref{ass: initialization},
    $ c_{1} \geq 2c_{1}^{\rm lin}, \rho_{1} \in \left[\frac{\rho_{1}^{\rm lin} + 1}{2}, 1\right), $ and $c_{2}\geq 2c_{2}^{\rm lin} $. If
    $
    \ell \lesssim {\frac{(1 - \rho_{1})^{7}(\sigma_{x}\sigma)^{2}}{(c_{1} + c_{2})c_{2}c_{1}^{7}(1 + \Gamma)^{8}D_{0}^{3}}}$ and   $
    \ell' \lesssim\frac{(1 - \rho_{1})^{8}(\sigma_{x}\sigma)^{2}}{(c_{1} + c_{2})^{2}c_{2}^{2}c_{1}^{16}(1 + \Gamma)^{6}D_{0}^{4}}
    $, then
    \begin{enumerate}[(a)]
        \item there exists a region $\Lambda(\delta) = \curly{K: \norm{K - K^{\rm lin}}_{F} \leq \delta}$ with  $\delta \asymp {\frac{(1 - \rho_{1})^{4}\sigma_{x}\sigma}{(c_{1}+c_{2})c_{1}^{6}\Gamma^{2}D_{0}}}$  such that $ \Lambda(\delta) \subset \Omega $ and $ \mcal{C}(K) $ is $ \mu $-strongly-convex and $ h $-smooth in $ \Lambda(\delta) $ with $ \mu = \sigma_{x}\sigma$ and $ h \asymp {\frac{\Gamma^{4}c_{1}^{4}D_{0}^{2}}{(1 - \rho_{1})^{2}}}$;
        \item the global minimum of $ \mcal{C}(K) $ is achieved at a point $ K^{*} \in \Lambda(\delta/3) $.

    \end{enumerate}
\end{Theorem}

Part (a) of Theorem \ref{thm: landscape} indicates that the cost function $\mcal{C}(K)$ is strongly convex and smooth within a $\delta$-neighborhood of the initializer $K^{\rm lin}$. Part~(b) shows that the optimal controller $K^{*}$ lies in a $\delta/3$-neighborhood of the initialization $K^{\rm lin}$. Consequently, the cost function is strongly convex and smooth in a region that contains both the initialization $K^{\rm lin}$ and the global optimizer $K^{*}$. These facts are crucial in establishing the global convergence of the proposed algorithm. {\color{black}We also remark that the bounds derived in Theorem \ref{thm: landscape} are only sufficient conditions and thus can be loose. Indeed, our numerical results in Section \ref{sec:experiments} show that the algorithm may still converge even when the Lipschitz constants are larger than the bounds required in Theorem 4.7.} The proof of Theorem \ref{thm: landscape} relies on Lemmas \ref{lemma:value_function}--\ref{lemma: bound Sigma_psi} which are detailed in Section \ref{sec: pf of local convexity}. {\color{black}Roughly speaking, Assumption \ref{ass:feature} implies $\norm{\nabla^2 \phi} \leq \ell'$ (assuming the second-order derivative exists). Consequently, we expect that $\norm{\nabla^2 C(K^*)} \geq \sigma - \ell' > 0$ when $\ell'$ is sufficiently small, which is the key idea for the proof.} 

\vspace{0.1cm}

\begin{proof}[Proof of Theorem \ref{thm: landscape}]
    We first show that $ \Lambda(\delta) \subset \Omega $ for any $ \delta \leq \min\curly{\frac{1 - \rho_{1}}{2\Gamma c_{1}^{\rm lin}}, \frac{c_{2}^{\rm lin}}{\Gamma}} $. Consider the following dynamics for $y_t \in \R^n$:
    \begin{align*}
    y_{t+1} = (A - BK_{1})y_{t} = (A - BK_{1}^{\rm lin})y_{t} + B(K_{1}^{\rm lin} - K_{1})y_{t}.
    \end{align*}
    Define a function $f: \R^n \to \R^n$ as $ f(y) = B(K_{1}^{\rm lin} - K_{1})y $. Simple algebraic manipulations show that $ f $ is $ \ell_{f} $-Lipschitz with $ \ell_{f} = \Gamma \delta $. Following the same argument in \cite[Lemma 4(a)]{Qu2020CombiningMA}, together with the assumption that $K \in \Omega$, we have
    \begin{subequations}
        \begin{align}
        \norm{y_{t}} & \leq 2c_{1}^{\rm lin}(\rho_{1}^{\rm lin} + 2c_{1}^{\rm lin}\ell_{f})^{t}\norm{y_{0}} \label{eqn:y_t a}
        \\ & \leq c_{1}(\rho_{1}^{\rm lin} + (1 - \rho_{1}))^t\norm{y_{0}} \label{eqn: y_t b}
        \\ & \leq c_{1}\rho_{1}^{t}\norm{y_{0}}. \label{eq: y_t c}
        \end{align}
        \end{subequations}
       Here, Eq.~\eqref{eqn:y_t a} follows \cite[Lemma 4(a)]{Qu2020CombiningMA}, Eq.~\eqref{eqn: y_t b} follows from the facts $ c_{1} \geq 2c_{1}^{\rm lin} $ and $ \delta \leq \frac{1- \rho_{1}}{2\Gamma c_{1}^{\rm lin}} $, and in Eq.~\eqref{eq: y_t c} we have used the fact $\rho_{1} \geq \frac{\rho_{1}^{\rm lin}+1}{2}$.
    Hence, we obtain 
    \begin{align}
    \norm{(A - BK_{1})^{t}} = \sup_{y_{0}\in \R^n}\frac{\norm{(A - BK_{1})^{t}y_{0}} }{\norm{y_{0}}} \leq c_{1}\rho_{1}^{t}. \label{eq:A - BK_1}
    \end{align}
    Additionally, since $ \norm{C - BK_{2}^{\rm lin}} \leq c_{2}^{\rm lin} $ and $ \norm{K_{2}^{\rm lin} - K^{2}} \leq \delta $, it follows
    \begin{align}
    \norm{C - BK_{2}} & = \norm{C- BK_{2}^{\rm lin} + B(K_{2}^{\rm lin} - K_{2})} \nonumber \\ & \leq \norm{C - BK_{2}^{\rm lin}} + \norm{B}\norm{K_{2}^{\rm lin} - K_{2}} \nonumber\\ & \leq c_{2}^{\rm lin} + \Gamma \delta  \leq 2c_{2}^{\rm lin} \leq c_{2}. \label{eq:C - BK_2}
    \end{align}
    Here, the penultimate inequality holds since $ \delta \leq c_{2}^{\rm lin}/\Gamma $ and the ultimate inequality follows from the fact that $ c_{2} \geq 2c_{2}^{\rm lin} $. Combining Eq.~\eqref{eq:A - BK_1} and \eqref{eq:C - BK_2}, we conclude that $ \Lambda(\delta) \subset \Omega $. 

    Next, we prove the local strong convexity of $ \mathcal{C}(K) $. 
    Let $ H = (A, C) \in\R^{n \times \parenthesis{n+d}}$ and $ \psi(x) = (x^{\top}, \phi(x)^{\top})^{\top} \in \R^{n+d} $. Since $ \phi $ is $ \ell $-Lipschitz, we know that $ \psi $ is $ \ell_{\psi} $-Lipschitz with $ \ell_{\psi} = \sqrt{1 + \ell^{2}} $. By Lemma \ref{lemma: grad of cost function}, the policy gradient satisfies
    \begin{align}
    \nabla \mcal{C}(K) = 2E_{K}\Sigma_{K}^{\psi\psi} - B^{\top}\Sigma_{K}^{G\psi}\label{eq:nabla C-1},
    \end{align}
    where we have defined
    \begin{align}
        E_{K}  & = RK - B^{\top}P_{K_{1}}(H - BK), \; \Sigma_{K}^{\psi\psi} = \E\bracket{\sum_{t = 0}^{\infty}\psi(x_{t})\psi(x_{t})^{\top}}, \nonumber \\  \; \text{and} \; \Sigma_{K}^{G\psi}  & = \E\bracket{\sum_{t = 0}^{\infty}\nabla G_{K}(x_{t+1})\psi(x_{t})^{\top}}. \label{eq:def-Ek}
    \end{align}
    Let $P_{K_1}$ satisfy
\begin{eqnarray}\label{eq:P}
(A-BK_1)^{\top}P_{K_1}(A-BK_1)-P_{K_1} +Q +K_1^{\top}R K_1=0,
\end{eqnarray}
and define $G_K(x)$  as
\begin{eqnarray}
 G_K(x) &:=&  \Tr \Big( \Big( K_2^{\top} R K_2 + (C-BK_2)^{\top}P_{K_1}(C-BK_2) \Big)\sum_{t=0}^{\infty} \phi(x_t)\phi(x_t)^{\top}\Big)\nonumber\\
&& + 2 \Tr \Big (\Big(K_1^{\top} R K_2+ (A-BK_1)^{\top}P_{K_1}(C-BK_2) \Big)\sum_{t=0}^{\infty}\phi(x_t)x_t^{\top} \Big).\label{eq:G_K}
\end{eqnarray}
Here, $\{x_t\}_{t=0}^{\infty}$ is the trajectory generated by the policy $K=(K_1,K_2)$ starting with the initial position $x_0$.
    By Lemma \ref{lemma: cost difference}, we have
    \begin{align}
        & \mcal{C}(K') - \mcal{C}(K) \nonumber\\ & = \Tr\parenthesis{(K' - K)^{\top}(R + B^{\top}P_{K_{1}}B)(K' - K)\Sigma_{K'}^{\psi\psi}} + 2\Tr\parenthesis{(K' - K)^{\top}E_{K}\Sigma_{K'}^{\psi\psi} }\nonumber\\ & \qquad +  \E\bracket{\sum_{t = 0}^{\infty}\bracket{G_{K}((H - BK')\psi(x_{t}')) - G_{K}((H - BK)\psi(x_{t}'))}}\nonumber \\ & = 2\Tr\parenthesis{(K' - K)^{\top}E_{K}\Sigma_{K}^{\psi\psi}} + 2\Tr\parenthesis{(K' - K)^{\top}E_{K}(\Sigma_{K'}^{\psi\psi} - \Sigma_{K}^{\psi\psi})}\nonumber \\ & \qquad + \Tr\parenthesis{(K' - K)^{\top}(R + B^{\top}P_{K_{1}}B)(K' - K)\Sigma_{K'}^{\psi\psi}} \nonumber\\ & \qquad + \E\bracket{\sum_{t = 0}^{\infty}\bracket{G_{K}((H - BK')\psi(x_{t}')) - G_{K}((H - BK)\psi(x_{t}'))}}. \label{eq:cvx-1}
    \end{align} 
    Moreover, since $ \phi $ is $ \ell $-Lipschitz with $ \ell \leq 1 $, Lemma \ref{lemma:trajectory_stability} implies
    \begin{align}
    \norm{(H - BK')\psi(x_{t'})} & = \norm{x_{t+1}'} \leq cD_{0}  \leq (c_{1} + c_{2})cD_{0},    \quad \text{and} \nonumber\\
    \norm{(H - BK)\psi(x_{t}')} & \leq \norm{(A - BK_{1})x_{t}'} + \norm{(C - BK_{2})\phi(x_{t}')} \leq (c_{1} + \ell c_{2})cD_{0} \leq (c_{1} + c_{2})cD_{0}. \label{eq: lip condition}
    \end{align}
    As $K \in \Lambda (\delta)$, Eq.~\eqref{eq: lip condition} implies that we can apply Lemma \ref{lemma:lip of G} to conclude
    \begin{align}
        & G_K((H - BK')\psi(x_t')) - G_K((H - BK)\psi(x_t')) \nonumber\\ & \geq (K - K')^{\top}B^{\top}\nabla G_K(x_{t+1}')\psi(x_t')^{\top}  - \frac{L}{2}\norm{B(K' - K)\psi(x_t')}_2^2, \label{eq:descent-1}
    \end{align}
    where we have defined
    \begin{align}
         L = \frac{5c_{2}c^{5}(1 + \Gamma)^{4}}{16(1 - \rho)^{2}}\ell + \frac{3Dc_{2}^{2}c^{6}(1 + \Gamma)^{2}}{16(1 - \rho)^{3}}\ell',\label{eq:L}
    \end{align}
    with $ D = (c_{1}+c_{2})c^{2}D_{0} $. Using Eq.~\eqref{eq:nabla C-1} and \eqref{eq:descent-1}, we can rewrite Eq.~\eqref{eq:cvx-1} as
    \begin{align}
        & \mcal{C}(K') - \mcal{C}(K) \nonumber
        \\ & \geq \Tr\parenthesis{(K' - K)^{\top}(2E_{K}\Sigma_{K}^{\psi\psi} - B^{\top}\Sigma_{K}^{G\psi})} + 2\Tr\parenthesis{(K' - K)^{\top}E_{K}(\Sigma_{K'}^{\psi\psi} - \Sigma_{K}^{\psi\psi})} \nonumber\\ & \qquad + \Tr\parenthesis{(K' - K)^{\top}(R + B^{\top}P_{K_{1}}B)(K' - K)\Sigma_{K'}^{\psi\psi}} \nonumber\\ & \qquad + \Tr\parenthesis{(K' - K)^{\top}B^{\top}\parenthesis{\E\bracket{\sum_{t = 0}^{\infty} \nabla G_{K}(x_{t+1})\psi(x_{t})^{\top}} - \E\bracket{\sum_{t = 0}^{\infty}\nabla G_{K}(x_{t+1}')\psi(x_{t}')^{\top}}}} \nonumber\\ & \qquad - \frac{L}{2}\E\bracket{\sum_{t = 0}^{\infty}\norm{B(K' - K)\psi(x_{t}')}^{2}} \nonumber\nonumber\\ & \geq \Tr\parenthesis{(K' - K)^{\top}\nabla \mcal{C}(K)} + \Tr\parenthesis{(K' - K)^{\top}(R + B^{\top}P_{K_{1}}B)(K' - K)\Sigma_{K'}^{\psi\psi}} \nonumber\\ & \qquad- 2\norm{K' - K}_{F}\norm{E_{K}}\norm{\Sigma_{K'}^{\psi\psi} - \Sigma_{K}^{\psi\psi}}_{F} \nonumber\\ & \qquad - \norm{K' - K}_{F}\norm{B}\norm{\E\bracket{\sum_{t = 0}^{\infty} \nabla G_{K}(x_{t+1})\psi(x_{t})^{\top}} - \E\bracket{\sum_{t = 0}^{\infty}\nabla G_{K}(x_{t+1}')\psi(x_{t}')^{\top}}}_{F} \nonumber\\ & \qquad - \frac{L}{2}\E\bracket{\sum_{t = 0}^{\infty}\norm{B}^{2}\norm{K' - K}_{F}^{2}\norm{\psi(x_{t}')}^{2}}. \label{eq:cvx-2}
    \end{align}
    Furthermore, one can see that
    \begin{align}
    & \Tr\parenthesis{(K' - K)^{\top}(R + B^{\top}P_{K_{1}}B)(K' - K)\Sigma_{K^{'}}^{\psi\psi}}\nonumber \\ & = \Tr\parenthesis{\parenthesis{(K' - K)(\Sigma_{K'}^{\psi\psi})^{1/2}}^{\top}(R + B^{\top}P_{K_{1}}B)(K' - K)(\Sigma_{K'}^{\psi\psi})^{1/2}} \label{eq:tr a} \\ & \geq \Tr\parenthesis{\parenthesis{(K' - K)(\Sigma_{K'}^{\psi\psi})^{1/2}}^{\top}(R + B^{\top}QB)(K' - K)(\Sigma_{K'}^{\psi\psi})^{1/2}} \label{eq:tr b}\\ & \geq \sigma\Tr\parenthesis{\parenthesis{(K' - K)(\Sigma_{K'}^{\psi\psi})^{1/2}}^{\top}(K' - K)(\Sigma_{K'}^{\psi\psi})^{1/2}} \label{eq:tr c}\\ & = \sigma\Tr \parenthesis{(K' - K)^{\top}\Sigma_{K'}^{\psi\psi}(K' - K)}  \geq \mu \norm{K' - K}_{F}^{2}, \label{eq: Tr lower bound}
    \end{align}
    where $ \mu = \sigma_{x}\sigma $. Eq.~\eqref{eq:tr a} holds because $ \Sigma_{K'}^{\psi\psi} \succ 0 $. Noting $ P_{K_{1}} \succeq Q $ and $ R + B^{\top}QB \succeq \sigma I $, we obtain Eq.~\eqref{eq:tr b} and \eqref{eq:tr c}. Finally, Eq.~\eqref{eq: Tr lower bound} holds due to the fact  $ \Sigma_{K'}^{\psi\psi} \succeq \E \bracket{\psi(x_{0})\psi(x_{0})^{\top}} \succeq \sigma_{x}I $. Combining Eq.~\eqref{eq:cvx-2} and \eqref{eq: Tr lower bound}, and applying Lemmas \ref{lemma:trajectory_stability} and \ref{lemma: bound Sigma},  it follows
    \begin{align*}
     \mcal{C}(K') - \mcal{C}(K) & \geq \Tr\parenthesis{(K' - K)^{\top}\nabla \mcal{C}(K)} + \mu\norm{K' - K}_{F}^{2} -2C_{1}C_{E}\norm{K - K^{\rm lin}}\norm{K' - K}_{F}^{2} \\ & \qquad  - \Gamma C_{2}\norm{K' - K}_{F}^{2} - \frac{L}{2}\Gamma^{2}\ell_{\psi}^{2}\norm{K' - K}_{F}^{2}\sum_{t = 0}^{\infty}\E\bracket{\norm{x_{t}'}^{2}} \\ & \geq \Tr\parenthesis{(K' - K)^{\top}\nabla \mcal{C}(K)} + \mu\norm{K' - K}_{F}^{2} - \bracket{2C_{1}C_{E}\delta + \frac{1}{2}\Gamma L C_{1} + \frac{L}{2}\frac{\Gamma^{2}\ell_{\psi}^{2}c^{2}D_{0}^{2}}{1 - \rho}}\norm{K' - K}_{F}^{2}, 
    \end{align*}
    where $ c = 2c_{1} $, $ \rho = \frac{\rho_{1}+1}{2} $,  and $ C_{E} $, $ C_{1}$ and $ C_{2}$ are defined as
    \begin{align} \label{eq:CE}
        C_{E} = 3(c_{1} + c_{2})\frac{\Gamma^{4}c^{3}}{(1 - \rho)^{2}}, C_{1} = \frac{4c^{3}\Gamma D_{0}^{2}}{( 1- \rho)^{2}},  C_{2} = LC_{1}/2 
    \end{align}
     To establish the local  strong convexity of $ \mcal{C}(\cdot) $, it remains to show that
    \begin{align}\label{eq:cvx-3}
        2C_{1}C_{E}\delta + \frac{1}{2}\Gamma L C_{1} + \frac{L}{2}\frac{\Gamma^{2}\ell_{\psi}^{2}c^{2}D_{0}^{2}}{1 - \rho} \leq \frac{\mu}{2}.
    \end{align}
    Notice by Lemma \ref{lemma:lip of G},  $ L = \ell C_{\ell} + \ell' C_{\ell'} $ with $ C_{\ell} = \frac{5c_{2}c^{5}(1 + \Gamma)^{4}}{16(1 - \rho)^{2}}$ and $ C_{\ell'} = \frac{3Dc_{2}^{2}c^{6}(1+\Gamma)^{2}}{16(1 - \rho)^{3}} $, where  $ D = (c_{1} + c_{2})c^{2}D_{0} $. Also, since $ \ell_{\psi} \leq \sqrt{2} $, we observe $ \frac{\Gamma \ell_{\psi}^{2}c^{2}D_{0}^{2}}{1 - \rho} \leq C_{1}  $. 
    Consequently, we conclude
    \begin{align}
        2C_{1}C_{E}\delta + \frac{1}{2}\Gamma L C_{1} + \frac{L}{2}\frac{\Gamma^{2}\ell_{\psi}^{2}c^{2}D_{0}^{2}}{1 - \rho}  \leq 2C_{1}C_{E}\delta + \frac{1}{2}\Gamma L C_{1} + \frac{L}{2}\Gamma C_{1} = 2C_{1}C_{E}\delta + \Gamma C_{1}(\ell C_{\ell} + \ell' C_{\ell'}) \leq \frac{\mu}{2}, \label{eq:cvx-4}
    \end{align}
    where the last inequality holds as long as
    \begin{align*}
    \delta  & \leq \frac{\mu}{12C_{1}C_{E}} = \frac{(1 - \rho)^{4}\sigma_{x}\sigma}{144(c_{1}+c_{2})c^{6}\Gamma^{2}D_{0}}, \\  \ell & \leq \frac{\mu}{6\Gamma C_{1}C_{\ell}} = \frac{(1 - \rho)^{4}\sigma_{x}\sigma}{30c_{2}c^{8}(1 + \Gamma)^{6}D_{0}^{2}}, \quad \text{and} \quad \\  \ell'  & \leq \frac{\mu}{6\Gamma C_{1}C_{\ell'}} = \frac{( 1- \rho)^{5}\sigma_{x}\sigma}{18(c_{1}+c_{2})c_{2}^{2}c^{11}(1 + \Gamma)^{4}D_{0}^{3}}.
    \end{align*}

    In a similar manner to the analysis of local strong convexity, we will demonstrate next that $ \mathcal{C}(K) $ is locally $ h $-smooth. First note Eq.~\eqref{eq: lip condition} and Lemma \ref{lemma:lip of G} imply
    \begin{align}
        & G_K((H - BK')\psi(x_t')) - G_K((H - BK)\psi(x_t')) \nonumber\\ & \leq (K - K')^{\top}B^{\top}\nabla G_K(x_{t+1}')\psi(x_t')^{\top}  + \frac{L}{2}\norm{B(K' - K)\psi(x_t')}_2^2, \label{eq:descent-2}
    \end{align}
    Then it follows from Lemma \ref{lemma: cost difference} that
    \begin{subequations}
        \begin{align}
        & \mcal{C}(K') - \mcal{C}(K)  \nonumber \\ & = \Tr\parenthesis{(K' - K)^{\top}(R + B^{\top}P_{K_{1}}B)(K' - K)\Sigma_{K'}^{\psi\psi}} \nonumber \\ & \qquad + 2\Tr\parenthesis{(K' - K)^{\top}E_{K}\Sigma_{K'}^{\psi\psi}} \nonumber\\ & \qquad +  \E\bracket{\sum_{t = 0}^{\infty}\bracket{G_{K}((H - BK')\psi(x_{t}')) - G_{K}((H - BK)\psi(x_{t}'))}}\nonumber \\ & = 2\Tr\parenthesis{(K' - K)^{\top}E_{K}\Sigma_{K}^{\psi\psi}} + 2\Tr\parenthesis{(K' - K)^{\top}E_{K}(\Sigma_{K'}^{\psi\psi} - \Sigma_{K}^{\psi\psi})} \nonumber \\ & \qquad + \Tr\parenthesis{(K' - K)^{\top}(R + B^{\top}P_{K_{1}}B)(K' - K)\Sigma_{K'}^{\psi\psi}} \nonumber\\ & \qquad + \E\bracket{\sum_{t = 0}^{\infty}\bracket{G_{K}((H - BK')\psi(x_{t}')) - G_{K}((H - BK)\psi(x_{t}'))}}. \label{eq:smooth a}
        \end{align}
        Applying Eq.~\eqref{eq:descent-2} to further upper bound
        \begin{align}
        \eqref{eq:smooth a} & \leq 2\Tr\parenthesis{(K' - K)^{\top}E_{K}\Sigma_{K}^{\psi\psi}} + 2\Tr\parenthesis{(K' - K)^{\top}E_{K}(\Sigma_{K'}^{\psi\psi} - \Sigma_{K}^{\psi\psi})}  \nonumber\\ & \qquad + \Tr\parenthesis{(K' - K)^{\top}(R + B^{\top}P_{K_{1}}B)(K' - K)\Sigma_{K'}^{\psi\psi}} \nonumber \\ & \qquad + \E\bracket{\sum_{t = 0}^{\infty}\Tr\parenthesis{(K - K')^{\top}B^{\top}\nabla G_{K}(x_{t+1}')\psi(x_{t}')^{\top}}}  \nonumber\\ & \qquad+ \E\bracket{\sum_{t = 0}^{\infty}\frac{L}{2}\norm{B(K' - K)\psi(x_{t}')}^{2}}\nonumber
        \\ & = \Tr\parenthesis{(K' - K)^{\top}\nabla \mcal{C}(K)} + 2\Tr\parenthesis{(K' - K)^{\top}E_{K}(\Sigma_{K'}^{\psi\psi} - \Sigma_{K}^{\psi\psi})}  \nonumber\\ & \qquad+ \Tr\parenthesis{(K' - K)^{\top}(R + B^{\top}P_{K_{1}}B)(K' - K)\Sigma_{K'}^{\psi\psi}} \nonumber \\ & \qquad + \E\left[\sum_{t = 0}^{\infty}\Tr\left((K' - K)^{\top}B^{\top}\left(\nabla G_{K}(x_{t+1})\psi(x_{t})^{\top} \right.\right.\right. \nonumber \\ & \left. \left.\left. \qquad\qquad\qquad\qquad\qquad\qquad\qquad\qquad\qquad - \nabla G_{K}(x_{t+1}')\psi(x_{t}')^{\top}\right)\right)\right] \nonumber  \\ & \qquad + \E\bracket{\sum_{t = 0}^{\infty}\frac{L}{2}\norm{B(K' - K)\psi(x_{t}')}^{2}}. \label{eq: smooth b}
        \end{align}
        The Cauchy-Schwarz inequality leads to
        \begin{align}
        \eqref{eq: smooth b} & \leq \Tr\parenthesis{(K' - K)^{\top}\nabla \mcal{C}(K)} + 2\norm{K' - K}_{F}\norm{E_{K}}\norm{\Sigma_{K'}^{\psi\psi} - \Sigma_{K}^{\psi\psi}}_{F} \nonumber \\ & \qquad + \norm{K' - K}_{F}^{2}\norm{R + B^{\top}P_{K_{1}}B}\norm{\Sigma_{K'}^{\psi\psi}}\nonumber \\ & \qquad + \norm{K' - K}_{F}\norm{B}\left\|\E\bracket{\sum_{t = 0}^{\infty}\nabla G_{K}(x_{t+1})\psi(x_{t})^{\top}} \right. \nonumber \\ & \left. \qquad\qquad\qquad\qquad\qquad\qquad\qquad\qquad- \E\bracket{\sum_{t = 0}^{\infty}\nabla G_{K}(x_{t+1}')\psi(x_{t}')^{\top}}\right\|_{F} \nonumber \\ & \qquad  + \E\bracket{\sum_{t = 0}^{\infty}\frac{L}{2}\norm{B}^{2}\norm{K' - K}_{F}^{2}\norm{\psi(x_{t}')}^{2}}. \label{eq: smooth c}
        \end{align}
        Next, we apply Lemmas \ref{lemma:trajectory_stability} and \ref{lemma: bound Sigma} to derive
        \begin{align}
        \eqref{eq: smooth c} & \leq \Tr\parenthesis{(K' - K)^{\top}\nabla \mcal{C}(K)} + \left(2C_{1}C_{E}\delta + \Gamma C_{2} + \frac{L}{2}\frac{\Gamma^{2}\ell_{\psi}^{2}c^{2}D_{0}^{2}}{1 - \rho} \right. \nonumber \\ & \qquad\qquad\qquad\qquad\qquad\quad\qquad  + \norm{R + B^{\top}P_{K_{1}}B}\norm{\Sigma_{K'}^{\psi\psi}}\Biggr)\norm{K' - K}_{F}^{2} \nonumber
       \\ & \leq \Tr(K' - K)^{\top}\nabla \mcal{C}(K) + \parenthesis{\frac{\mu}{2} + \norm{R + B^{\top}P_{K_{1}}B}\norm{\Sigma_{K'}^{\psi\psi}}}\norm{K' - K}_{F}^{2}, \label{eq: smooth d}
    \end{align}
    \end{subequations}
    where Eq.~\eqref{eq: smooth d} is a consequence of Eq.~\eqref{eq:cvx-3}.
    Finally, applying the upper bounds on $ \norm{P_{K_{1}}} $ and $ \norm{\Sigma_{K'}^{\psi\psi}} $ from Lemmas \ref{lemma: bound P_K} and  \ref{lemma: bound Sigma_psi},  we obtain
    \begin{align}
    \mu + 2\norm{R + B^{\top}P_{K_{1}}B}\norm{\Sigma_{K'}^{\psi\psi}}  & \leq \mu + 2\parenthesis{1 + \Gamma^{2}\frac{c^{2}\Gamma^{2}}{1 - \rho}}\frac{2c^{2}D_{0}^{2}}{1 - \rho} \nonumber \\ & \leq 9\frac{\Gamma^{4}c^{4}D_{0}^{2}}{(1 - \rho)^{2}} \eqqcolon h. \label{eq: h}
    \end{align}
    Therefore, combining Eq.~\eqref{eq:cvx-4} and \eqref{eq: h}, we finish the proof of part (a) of Theorem \ref{thm: landscape}.  

    In the following, we will prove part (b) of Theorem \ref{thm: landscape}. We first observe that $ E_{K^{\rm lin}} = RK^{\rm lin} - B^{\top}P_{K_1^{\rm lin}}(H - BK^{\rm lin}) = 0 $ by Eq.~\eqref{eq:linear_initialization}--\eqref{eq:nonlinear_initialization}. Then it follows from Lemma \ref{lemma: cost difference} that 
    \begin{align}
    \mcal{C}(K) - \mcal{C}(K^{\rm lin}) & = 2\Tr\parenthesis{(K - K^{\rm lin})^{\top}E_{K^{\rm lin}}\Sigma_{K}^{\psi\psi}} + \Tr\parenthesis{(K - K^{\rm lin})^{\top}(R + B^{\top}P_{K_{1}^{\rm lin}}B)(K - K^{\rm lin})\Sigma_{K}^{\psi\psi}} \nonumber\\ & \qquad + \E\bracket{\sum_{t = 0}^{\infty}\bracket{G_{K^{\rm lin}}((H - BK))\psi(x_{t}) - G_{K^{\rm lin}}((H - BK^{\rm lin}))\psi(x_{t})}} \nonumber\\ & = \Tr\parenthesis{(K - K^{\rm lin})^{\top}(R + B^{\top}P_{K_{1}^{\rm lin}}B)(K - K^{\rm lin})\Sigma_{K}^{\psi\psi}} \nonumber\\ & \qquad + \E\bracket{\sum_{t = 0}^{\infty}\bracket{G_{K^{\rm lin}}((H - BK))\psi(x_{t}) - G_{K^{\rm lin}}((H - BK^{\rm lin}))\psi(x_{t})}}. \label{eq:opt-1}
    \end{align}
    Additionally,  Eq.~\eqref{eq: Tr lower bound} with $ K' = K^{\rm lin} $ implies that
    \begin{align*}
        \Tr(K - K^{\rm lin})^{\top}(R + B^{\top}P_{K_{1}^{\rm lin}}B)(K - K^{\rm lin})\Sigma_{K}^{\psi\psi}  \geq \sigma_{x}\sigma \norm{K - K^{\rm lin}}_{F}^{2} = \mu\norm{K - K^{\rm lin}}_{F}^{2}.
    \end{align*}
    Also, note that
    \begin{align*}
    \norm{(H - BK)\psi(x_{t})} & = \norm{x_{t+1}} \leq c\norm{x_{0}} \leq cD_{0} \leq (c_{1} + c_{2})cD_{0}, \quad \text{and} \\ \norm{(H - BK^{\rm lin})\psi(x_{t})} & = \norm{(A - BK_{1}^{\rm lin})x_{t} + (C - BK_{2}^{\rm lin})\phi(x_{t})} \leq (c_{1} + \ell c_{2})cD_{0} \leq (c_{1} + c_{2})cD_{0}.
    \end{align*}
    Hence, we can apply Lemma \ref{lemma:lip of G} to obtain
    \begin{align*}
       &  G_{K^{\rm lin}}((H - BK))\psi(x_{t}) - G_{K^{\rm lin}}((H - BK^{\rm lin}))\psi(x_{t}) \nonumber\\ & \geq -\Tr\parenthesis{(B(K - K^{\rm lin})\psi(x_{t}))^{\top}\nabla G_{K^{\rm lin}}((H - BK)\psi(x_{t}))} - \frac{L}{2}\norm{B(K - K^{\rm lin})\psi(x_{t})}^{2} \\ & \geq - \norm{B}\norm{K - K^{\rm lin}}_{F}\norm{\psi(x_{t})}L\norm{x_{t+1}} - \frac{L}{2}\norm{B}^{2}\norm{K - K^{\rm lin}}_{F}^{2}\norm{\psi(x_{t})}^{2}, 
    \end{align*}
    where we have used the fact that $ \nabla G_{K^{\rm lin}}(0) = 0 $ to reach the second inequality. 
    As such, we can deduce
    \begin{align}
        & \E\bracket{\sum_{t = 0}^{\infty}\bracket{G_{K^{\rm lin}}((H - BK))\psi(x_{t}) - G_{K^{\rm lin}}((H - BK^{\rm lin}))\psi(x_{t})}} \nonumber \\ & \geq -L\Gamma\ell_{\psi}\norm{K - K^{\rm lin}}_{F}\sum_{t = 0}^{\infty}\E\bracket{\norm{x_{t}}^{2}} - \frac{L}{2}\Gamma^{2}\ell_{\psi}^{2}\norm{K - K^{\rm lin}}_{F}^{2}\sum_{t = 0}^{\infty}\E\bracket{\norm{x_{t}}^{2}} \label{eq: G a} \\ & \geq -\frac{\rho L\Gamma \ell_{\psi}c^{2}D_{0}^{2}}{1 - \rho^{2}}\norm{K - K^{\rm lin}}_{F} - \frac{L}{2}\frac{\Gamma^{2}\ell_{\psi}^{2}c^{2}D_{0}^{2}}{1 - \rho^{2}}\norm{K - K^{\rm lin}}_{F}^{2} \label{eq: G b}   \\ & \geq -\frac{ L\Gamma \ell_{\psi}c^{2}D_{0}^{2}}{1 - \rho}\norm{K - K^{\rm lin}}_{F} - \frac{L}{2}\frac{\Gamma^{2}\ell_{\psi}^{2}c^{2}D_{0}^{2}}{1 - \rho}\norm{K - K^{\rm lin}}_{F}^{2}. \label{eq:descent-3}
    \end{align}
    Here, Eq.~\eqref{eq: G a} holds since $ \norm{B} \leq \Gamma $ and $ \norm{\psi(x_{t})} \leq \ell_{\psi}\norm{x_{t}} $, Eq.~\eqref{eq: G b} follows from $ \norm{x_{t}} \leq c\rho^{t}D_{0} $, and we have used the fact that $1/2 < \rho < 1$ to obtain Eq \eqref{eq:descent-3}.    
    By leveraging Eq.~\eqref{eq:descent-3}, we can rewrite Eq.~\eqref{eq:opt-1} as
    \begin{align*}
    \mcal{C}(K) - \mcal{C}(K^{\rm lin}) \geq \bracket{\mu - \frac{L}{2}\cdot\frac{\Gamma^{2}\ell_{\psi}^{2}c^{2}D_{0}^{2}}{1 - \rho}}\norm{K - K^{\rm lin}}_{F}^{2} - \frac{L\Gamma\ell_{\psi}c^{2}D_{0}^{2}}{1 - \rho}\norm{K - K^{\rm lin}}_{F}.
    \end{align*}
    Since $ \norm{K - K^{\rm lin}}_{F} > \delta / 3 $, it suffices to show that
    \begin{align*}
        \bracket{\mu - \frac{L}{2}\cdot\frac{\Gamma^{2}\ell_{\psi}^{2}c^{2}D_{0}^{2}}{1 - \rho}}\norm{K - K^{\rm lin}}_{F} - \frac{L\Gamma\ell_{\psi}c^{2}D_{0}^{2}}{1 - \rho}\geq 0.
    \end{align*}
    Indeed, this inequality holds since
    \begin{align*}
        \frac{L}{2}\cdot\frac{\Gamma^{2}\ell_{\psi}^{2}c^{2}D_{0}^{2}}{1 - \rho} \leq \frac{\mu}{2}, \quad\text{and}\quad
        \frac{L\Gamma\ell_{\psi}c^{2}D_{0}^{2}}{1 - \rho} \leq \frac{\mu \delta}{6},
    \end{align*}
    as long as the following conditions are satisfied
    \begin{align*}
    \ell & \leq \delta\frac{2(1 - \rho)^{3}\sigma_{x}\sigma}{9c_{2}c^{7}(1 + \Gamma)^{6}D_{0}^{2}}, \quad \text{and} \quad \ell' \leq \delta\frac{2(1 - \rho)^{4}\sigma_{x}\sigma}{9(c_{1}+c_{2})c_{2}^{2}c^{10}(1 + \Gamma)^{4}D_{0}^{3}}.
    \end{align*}
    Choose $ \delta $ as in the proof of part (a), we finish the proof of part (b) of Theorem \ref{thm: landscape}. 
\end{proof}

Given the landscape results, if $\nabla \mcal{C}(K)$ is assumed to be known, starting from the initialization $K^{\rm lin}$, the policy gradient method leads to  the global minimum of the cost function $\mcal{C}(K)$. Hence, it is not surprising that the policy gradient method \eqref{eq:algo} converges to the globally optimal solution with the gradient estimation {in Algorithm \ref{alg:policy gradient estimation}}. This result is formally stated in Theorem \ref{thm: conv of algo}. Recall the policy update rule $K^{(m+1)} = K^{(m)} - \eta \widehat{\nabla \mathcal{C}(K^{(m)})}$.

\begin{Theorem}\label{thm: conv of algo}
      Assume the conditions in Theorem \ref{thm: landscape} hold. Let $ \epsilon > 0 $ and $ \nu \in (0, 1) $ be given. Suppose the step size $ \eta < \frac{1}{h} $ and the number of gradient descent steps $ M \geq  \frac{2}{\eta \mu}\log \parenthesis{\frac{\delta}{3}\sqrt{\frac{2h}{\epsilon}}} $. Further, assume the gradient estimator parameter {in Algorithm \ref{alg:policy gradient estimation}} satisfies $ r \leq \min\curly{\frac{\delta}{3}, \frac{1}{3h}e_{\text{grad}}} $, $ T \geq \frac{1}{1 - \rho_{1}}\log\frac{6\widehat{D}C_{\text{max}}}{e_{\text{grad}}r}$, and
      \begin{align*}
      J \geq \frac{\widehat{D}^{2}}{e_{\text{grad}}^{2}r^{2}}\log \frac{4\widehat{D}M}{\nu}\max\curly{36\parenthesis{\mcal{C}(K^{*}) + 2h\delta^{2}}^{2}, 144C_{\max}^{2}},
      \end{align*}
    where $ C_{\max} = \frac{24(1+\Gamma)^{2}c_1^{2}D_{0}^{2}}{1 - \rho_1}  $, and $ e_{\text{grad}} = \min\curly{\frac{\delta \mu}{6}, \frac{\mu}{2}\sqrt{\frac{\epsilon}{2h}}} $. Then with probability at least $ 1 - \nu $, we have $ \mcal{C}(K^{(M)}) - \mcal{C}(K^{*}) < \epsilon $. 
\end{Theorem}

This result shows that, despite the existence of nonlinear terms, finding the optimal control policy is still tractable when nonlinear terms are ``sufficiently small''. Moreover, we comment that the convergence rate $ \mathcal{O}\parenthesis{\frac{h}{\mu}\log\parenthesis{\frac{1}{\epsilon}}} $ matches that of LQR \cite{fazel2018global} in terms of the dependency on $ \mu $ and $ \epsilon $. {\color{black}Furthermore, the policy gradient approach requires a total number of $ MJT $ samples to perform the gradient estimation for $ M $ times. Here, the dependency of parameters in Algorithm \ref{alg:policy gradient estimation} are given by $ r = \mcal{O}\parenthesis{\frac{\mu}{h}\sqrt{\frac{\epsilon}{h}}} $, $ T = \mcal{O}\parenthesis{\log\parenthesis{\frac{h^{2}}{\epsilon \mu^{2}}}} $ and $ J = \tilde{\mcal{O}}\parenthesis{\frac{h^{4}}{\epsilon^{2}\mu^{4}}} $. Finally, we remark that our zeroth-order optimization framework is one of many possibilities of policy-based methods. One can improve the sample complexity by incorporating variance reduction techniques into our framework.} The proof of Theorem \ref{thm: conv of algo} relies on Lemma \ref{lemma: estimation C(K)} which is deferred to Section \ref{sec: conv of algo}.

\begin{proof}[Proof of Theorem \ref{thm: conv of algo}]
    Let $ \mcal{F}_{m} $ be the filtration generated by $ \curly{\widehat{\nabla \mcal{C}(K^{(m')})}}_{m' = 0}^{m - 1} $. Define the following event:
    \begin{align*}
        \mcal{E}_{m} &  = \curly{K^{(m')} \in \text{Ball}(K^{*}, \delta/3), m' = 0, \dots, m} \cap \curly{\norm{\widehat{\nabla \mcal{C}}(K^{(m')}) - \nabla \mcal{C} (K^{(m')})}_{F}\leq e_{\text{grad}}, m' = 0, \dots, m - 1},
    \end{align*}
    where $ \text{Ball}(K^{*}, \delta/3) = \curly{K: \norm{K - K^{*}}_{F} \leq \delta/3} $. Apparently, both $ K^{(m)} $ and the event $ \mcal{E}_{m} $ are $ \mcal{F}_{m} $-measurable. We want to show the following inequality:
    \begin{align}\label{eqn: high prob event}
    \E\bracket{1(\mcal{E}_{m+1}) \vert \mcal{F}_{m}}1(\mcal{E}_{m}) \geq \parenthesis{1 - \frac{\nu}{M}}1(\mcal{E}_{m}).
    \end{align}
    Namely, if event $ \mcal{E}_{m} $ is true, conditioned on $ \mcal{F}_{m} $, the event $ \mcal{E}_{m+1} $ happens with probability at least $ 1-\nu/M$. Note that conditioned on event $ \mcal{E}_{m} $, we have $ \norm{K^{(m)} - K^{\rm lin}}_{F} \leq \norm{K^{(m)} - K^{*}}_{F} + \norm{K^{\rm lin} - K^{*}}_{F} \leq 2\delta/3 $, which follows that $ K^{(m)} \in \Lambda(2\delta/3)  $.  Next, we show that $ K^{(m+1)} \in \text{Ball}(K^{*}, \delta/3) $. 
    Note that by $ \mu $-strong convexity of the cost function $ \mcal{C} $, it holds
    \begin{align}
    & \norm{K^{(m)} - \eta \nabla\mcal{C}(K^{(m)}) - K^{*}}_{F}^{2} \nonumber \\ & = \norm{K^{(m)} - K^{*}}_{F}^{2} - 2\eta\Tr\parenthesis{\nabla\mcal{C}(K^{(m)})^{\top}(K^{(m)} - K^{*})} + \eta^{2}\norm{\nabla \mcal{C}(K^{(m)})}_{F}^{2} \nonumber \\ & \leq (1 - \eta \mu)\norm{K^{(m)} - K^{*}}_{F}^{2} - 2\eta\parenthesis{\mcal{C}(K^{(m)}) - \mcal{C}(K^{*})} + \eta^{2}\norm{\nabla \mcal{C}(K^{(m)})}_{F}^{2}. \label{eq:strong-cvx}
    \end{align}
    Furthermore, since $ \mcal{C}(\cdot) $ is $ h $-smooth, we have
    \begin{align*}
    \mcal{C}(K^{(*)}) - \mcal{C}(K^{(m)}) & \leq \mcal{C}\parenthesis{K^{(m)} - \frac{1}{h}\nabla \mcal{C}(K^{m})} - \mcal{C}(K^{(m)})\\ & \leq -\frac{1}{h}\Tr\parenthesis{\nabla \mcal{C}(K^{(m)})^{\top}\nabla \mcal{C}(K^{(m)})} + \frac{h}{2}\norm{\frac{1}{h}\nabla \mcal{C}(K^{(m)})}_{F}^{2} = -\frac{1}{2h}\norm{\nabla \mcal{C}(K^{(m)})}_{F}^{2}.
    \end{align*}
    Thus, Eq.~\eqref{eq:strong-cvx} becomes
    \begin{align}
        \norm{K^{(m)} - \eta \nabla\mcal{C}(K^{(m)}) - K^{*}}_{F}^{2} & \leq (1 - \eta \mu)\norm{K^{(m)} - K^{*}}_{F}^{2} - 2\eta(1 - h\eta)\parenthesis{\mcal{C}(K^{(m)}) - \mcal{C}(K^{*})}\nonumber \\ & \leq (1 - \eta \mu)\norm{K^{(m)} - K^{*}}_{F}^{2}, \label{eq:smooth-bound}
    \end{align}
    where $ \eta(1 - h\eta) > 0 $ since $ 0 < \eta < 1/h $. Note under our selection of parameters, by Lemma \ref{lemma: estimation C(K)}, we have $ \norm{\widehat{\nabla \mcal{C}}(K^{(m)}) - \nabla \mcal{C} (K^{(m)})}_{F} \leq e_{\text{grad}} $ with probability at least $ 1 - \nu/M $. Together with Eq.~\eqref{eq:smooth-bound}, with probability at least $ 1 - \nu/M $, we have
    \begin{align}
    \norm{K^{(m+1)} - K^{*}}_{F} & \leq \norm{K^{(m)} - \eta \nabla \mcal{C}(K^{(m)}) - K^{*}}_{F} + \eta \norm{\widehat{\nabla \mcal{C}}(K^{(m)}) - \nabla \mcal{C} (K^{(m)})}_{F} \nonumber\\ & \leq ( 1- \eta \mu)^{1/2}\norm{K^{(m)} - K^{*}}_{F} + \eta e_{\text{grad}} \label{eq:Km} \\ & \leq ( 1- \eta \mu)^{1/2}\frac{\delta}{3} + \eta e_{\text{grad}} \label{eq: k(m+1) a} \\ & \leq \parenthesis{1 - \frac{1}{2}\eta\mu}\frac{\delta}{3} + \eta \frac{\delta \mu}{6}  = \frac{\delta}{3}. \label{K(m+1) b}
    \end{align}
    Here, Eq.~\eqref{eq: k(m+1) a} follows from the fact $ \norm{K^{(m)} - K^{*}}_{F} \leq \frac{\delta}{3} $. 
    We have used the facts $ (1 - x)^{1/2} \leq 1 - \frac{1}{2}x $ and $ e_{\rm grad} \leq \frac{\delta \mu}{6} $ to derive Eq.~\eqref{K(m+1) b}.
    As such, taking the expectation of \eqref{eqn: high prob event} on both sides, we have
    \begin{align}\label{eq:recursion}
    \P(\mcal{E}_{m+1}) = \P(\mcal{E}_{m+1} \cap \mcal{E}_{m}) = \E\bracket{\E\bracket{1(\mcal{E}_{m+1}) \vert \mcal{F}_{m}}1(\mcal{E}_{m})} \geq \parenthesis{1 - \frac{\nu}{M}}\P(\mcal{E}_{m}).
    \end{align}
    Unrolling Eq.~\eqref{eq:recursion}, we obtain $ \P(\mcal{E}_{M}) \geq \parenthesis{1 - \frac{\nu}{M}}^{M}\P(\mcal{E}_{0}) = \parenthesis{1 - \frac{\nu}{M}}^{M} \geq 1 - \nu $. 
    Now, on event $ \mcal{E}_{M} $, by Eq.~\eqref{eq:Km}, we also have
    \begin{align}
    \norm{K^{(M)} - K^{*}}_{F} & \leq (1 - \eta\mu)^{M/2}\norm{K^{(0)} - K^{*}}_{F} + \eta e_{\text{grad}}\sum_{m = 0}^{M-1}(1 - \eta \mu)^{m/2} \nonumber\\ & \leq (1 - \eta\mu)^{M/2}\frac{\delta}{3} + \eta e_{\text{grad}}\sum_{m = 0}^{\infty}(1 - \eta \mu)^{m/2}\label{eq:M a} \\ & \leq (1 - \eta \mu)^{M/2}\frac{\delta}{3} + \frac{2e_{\text{grad}}}{\mu}  \leq \sqrt{\frac{2\epsilon}{h}}. \label{eq:M}
    \end{align} 
    Here, Eq.~\eqref{eq:M a} holds since $\norm{K^{(0)} - K^{*}} \leq \delta/3$, and Eq.~\eqref{eq:M} follows from the assumptions that $ M \geq \frac{2}{\eta \mu}\log \parenthesis{\frac{\delta}{3}\sqrt{\frac{2h}{\epsilon}}} $ and $ e_{\text{grad}} \leq \frac{\mu}{2} \sqrt{\frac{\epsilon}{2h}} $. Finally, by the $ h $-smoothness of $ \mcal{C}(\cdot) $ again, we conclude that with probability at least $ 1 - \nu $, 
    \begin{align*}
    \mcal{C}(K^{(M)}) \leq \mcal{C}(K^{*}) + \frac{h}{2}\norm{K^{(M)} - K^{*}}_{F}^{2} \leq \mcal{C}(K^{*}) + \epsilon,
    \end{align*}
    which finishes the proof. 
\end{proof}

{\color{black}
\section{Numerical Experiments}
In this section, we numerically evaluate the performance of our policy gradient method proposed in Section \ref{sec: algorithm} through extensive experiments. %
In particular, we  focus on addressing the following questions:
\begin{itemize}
    \item In practice, how fast does the policy gradient algorithm with {\it known model parameters} converge to the optimal solution? How sensitive is the policy gradient algorithm to the initialization?
    \item Does the policy gradient algorithm still converge when the Lipschitz continuity assumption in Theorem \ref{thm: conv of algo} is violated? How restrictive is the condition in practice?
\end{itemize} 
As we will see in this section, our policy gradient algorithm converges to the globally optimal solution and is robust to the magnitude of the nonlinear term and the policy initialization regimes.

\paragraph{Model and Parameter Setup}We experiment on (randomly generated) synthetic data. Specifically, we set $ n $ (the dimension of state), $ p $ (the dimension of control) and $ d $ (the dimension of kernel basis) to be $3$. The cost is set to be $ Q = R = I_{3 \times 3} $. The matrices $ A $, $ B $ and $ C $ are generated randomly, with each entry drawn from a standard Gaussian distribution. The model parameters are normalized such that the spectral radius is less than $ 1 $ with high probability. The kernel basis is fixed to be $ \phi(x) = \ell\sin(x) $, where the operations are understood as entrywise and $ \ell $ is the Lipschitz constant of the nonlinear term. The initial distribution $ \mathcal{D} $ of $ x_{0} $ is chosen as a standard Gaussian and $\norm{x_0}$ is rescaled to be 1. 

\paragraph{Evaluation}To study the convergence of the policy gradient method, we consider two different settings. For the first setting, we fix $ \ell = 1 $ and choose three initialization regimes. In Figure \ref{fig:pg-init}, $ K^{\rm lin} $ is computed by \eqref{eq:linear_initialization}--\eqref{eq:nonlinear_initialization}, where $ P $ is obtained by solving the ARE \eqref{eq:ARE} as $ A $, $ B $ and $ C $ are assumed to be known. Also, the random policy $ K^{\rm rand} $ is generated by drawing a matrix of size $ p \times (n+d) $ from the unit sphere (in 2-norm) uniformly at random. The gradient estimate is constructed by Algorithm \ref{alg:policy gradient estimation} with parameters $ J = 300 $, $ T = 10 $ and $ r = 0.6 $. To measure the performance of each policy, we empirically evaluate the cost function by sampling $ J $ trajectories with the same $ T $ in each trajectory. We perform the gradient descent step for 200 iterations and choose the step size $ \eta = 10^{-4} $. In the second experiment, the initial policy is fixed to be $ K^{(0)}=K^{\rm lin} $. We vary the Lipschitz constant $ \ell $ from $ 1 $ to $ 6 $ and report the cost across iterations in Figure \ref{fig:pg-ell}. {\color{black}Moreover, we demonstrate the robustness of our algorithm by varying the random seeds for model parameter generation (see Figure \ref{fig:robust}).  In this experiment, the model parameters $A$, $B$ and $C$ are randomly generated, while the Lipschitz constant is fixed as $\ell = 3$ and the initial policy is set to be $K^{\rm lin}$. }

\begin{figure}[H]
    \centering
    \begin{subfigure}[H]{0.49\linewidth}
        \centering
        \includegraphics[width=\linewidth]{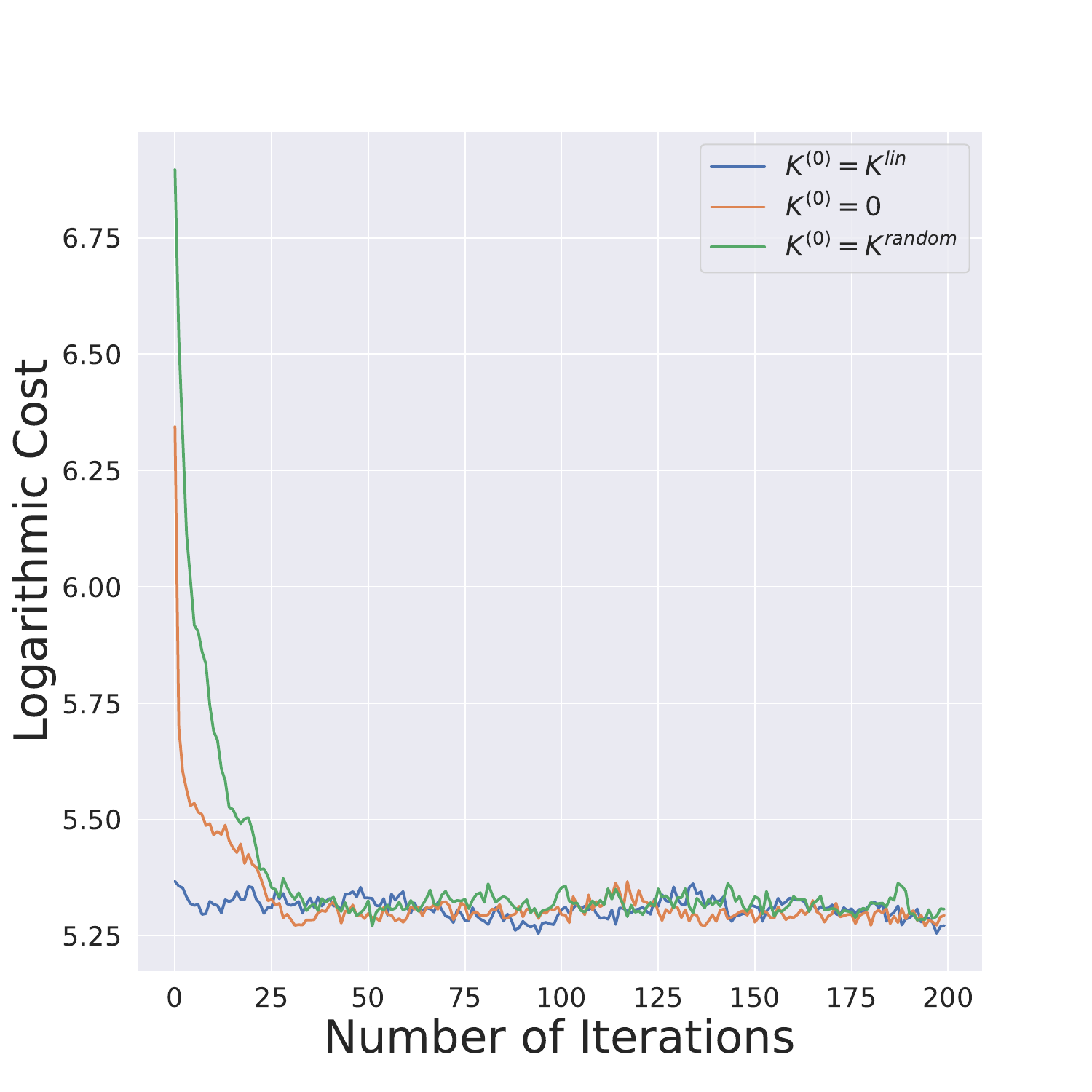}
        \caption{Impact of initialization}
        \label{fig:pg-init}
    \end{subfigure}
    \hfill
    \begin{subfigure}[H]{0.49\linewidth}
        \centering
        \includegraphics[width=\linewidth]{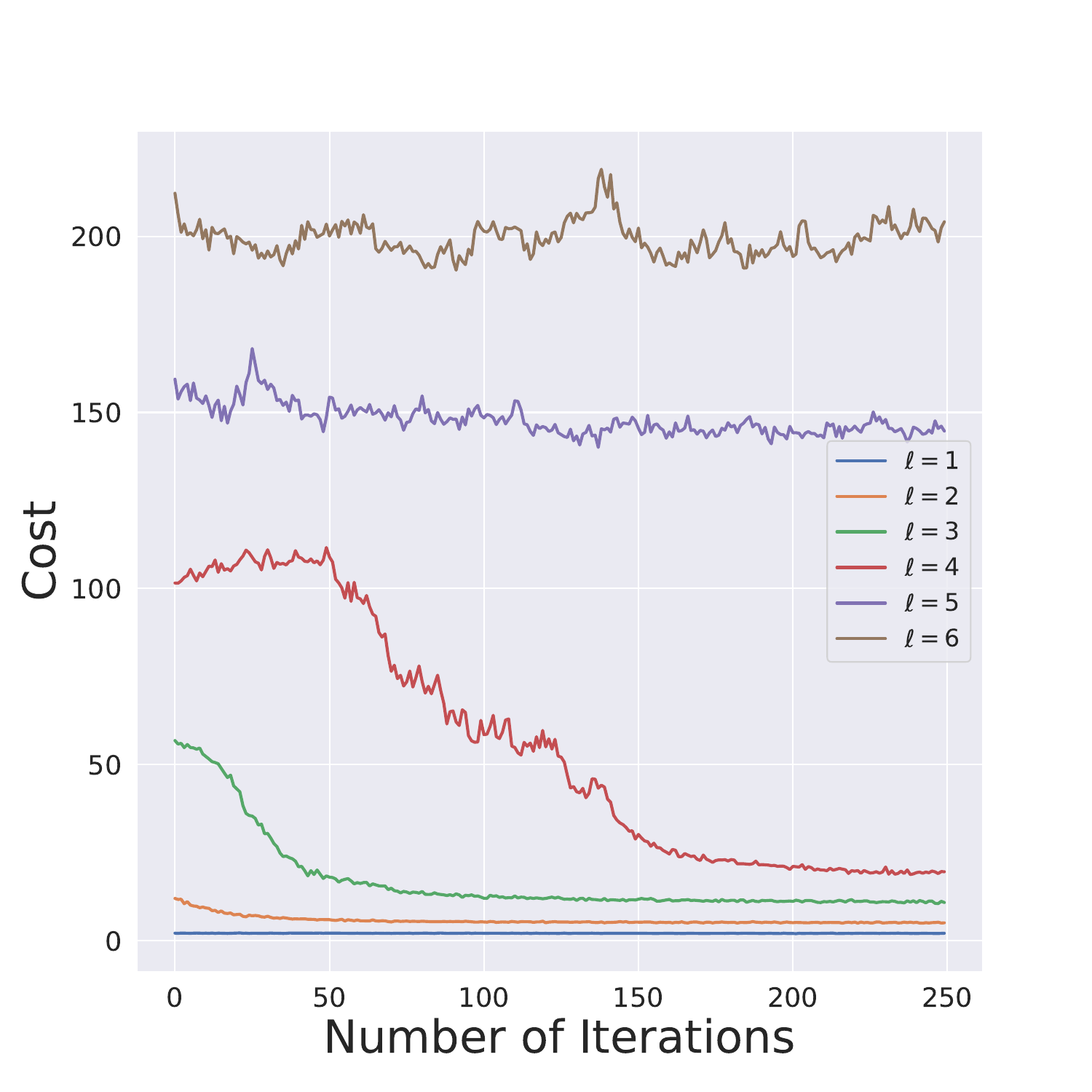}
        \caption{Impact of Lipschitz constants}
        \label{fig:pg-ell}
    \end{subfigure}
    \caption{Convergence of the policy gradient algorithm}
\end{figure}

\paragraph{Discussion}In Figure \ref{fig:pg-init}, we observe that the policy gradient algorithm converges under all three initialization regimes, with promising accuracy achieved within around 50 iterations.  This indicates that the algorithm is relatively stable with small fluctuations and is consistent with the linear convergence rate demonstrated in the theoretical part. We also observe that the initial value obtained by the policy $ K^{\rm lin} $ is comparably close to its convergent value. Such a phenomenon implies that $ K^{\rm lin} $ is close to the optimal solution $ K^{*} $ as expected. 

In Figure \ref{fig:pg-ell}, we observe that the policy gradient algorithm converges when $ \ell \leq 4 $ and the method does not converge for $ \ell \geq 5 $. {\color{black}Furthermore, Figure \ref{fig:robust} suggests the convergence of our policy gradient method under numerous model configurations regardless of the non-linear system dynamics. } Therefore, we conclude that the algorithm is robust within a certain magnitude of the nonlinear term, and extends to cases beyond the theoretical requirements.

}

\begin{figure}[H]
\centering
\includegraphics[width=0.6\linewidth]{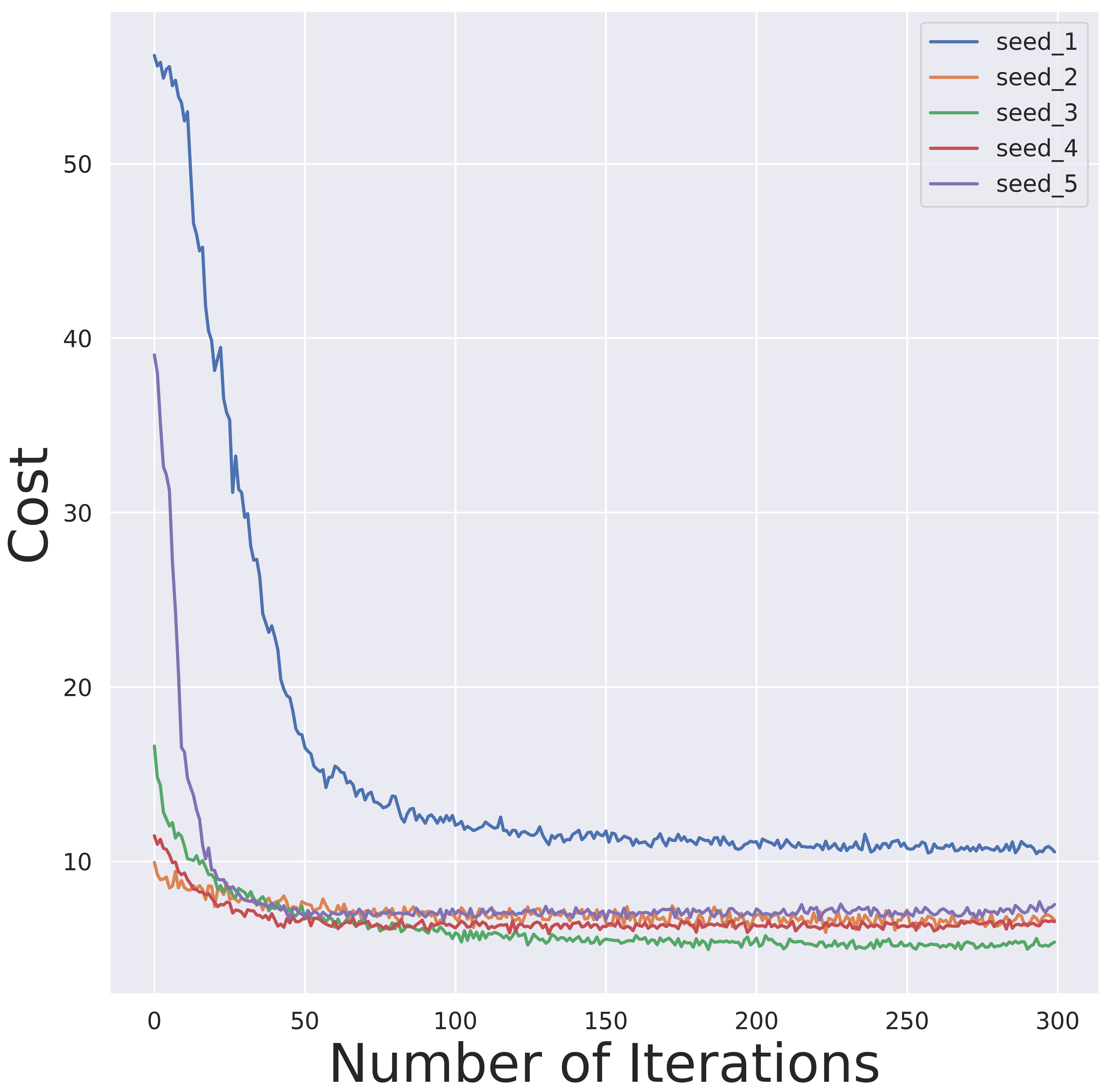}
\caption{Robustness of policy gradient algorithm.}
\label{fig:robust}
\end{figure}

\section{Proofs}\label{sec: proofs}
In this section, we prove several technical lemmas that are used in Section \ref{sec: main results}. 
\subsection{Proof of Proposition \ref{prop: inv of design matrix}}\label{sec: proof of matrix inv}
We start by presenting a standard non-asymptotic bound on the minimum singular value of a random matrix with sub-Gaussian rows. Denote $\lambda_{\min}(Y^\top Y)$ as the smallest singular value of a matrix $Y$. 
\begin{Lemma}[{\cite[Theorem 5.39]{vershynin2010introduction}}] \label{lemma: sub-gaussian rows}
    Let $ Y \in \R^{N \times k} $ be a matrix whose rows are independent sub-Gaussian isotropic random vectors in $ \R^{k} $. Then for every $ \nu \in (0, 1)$, with probability at least $ 1 - \nu $, one has 
    \begin{align*}
        \sqrt{\lambda_{\min}(Y^{\top}Y)} \geq \sqrt{N} - d_1\sqrt{k} - \sqrt{\frac{1}{d_2}\log\frac{2}{\nu}}.
    \end{align*}
    Here $ d_{1}, d_{2} $ are absolute constants that only depend on the sub-Gaussian norm of the rows. In particular, $ d_{1} = 1, d_{2} = 1/2 $ if $Y$ has i.i.d. $\mathcal{N}(0, 1)$ entries.
\end{Lemma}

The next result shows that the second-moment matrix of the row vectors of $ \Phi_{N}^{\top}\Phi_{N} $ is invertible. Recall $ \Sigma = \E\bracket{\varphi^{(i)}\parenthesis{\varphi^{(i)}}^{\top}} $ with $ \varphi^{(i)} = \varphi(x_{0}^{(i)}, u_{0}^{(i)}) $.

\begin{Lemma}\label{lemma: inv of sed mom matrix}
    Assume Assumption \ref{ass:initial_distr} holds. If the random vector $ u_{0}^{(i)} \in \R^{p} $ is such that $ 0 < \P\parenthesis{\abs{w^{\top}u_{0}^{(i)}} > 0} < 1 $ for all $ w \neq 0 $, then the matrix $ \Sigma $ is invertible.
\end{Lemma}

\begin{proof}
    Since $\Sigma$ is a symmetric matrix, it is equivalent to show that $ \Sigma $ is positive definite. 
    Let $ s = (s_{1}^{\top}, s_{2}^{\top}, s_{3}^{\top})^{\top} \neq 0 $. The matrix $ \Sigma $ is positive definite if and only if
    \begin{align*}
    s^{\top}\Sigma s = \E \bracket{\abs{s^{\top}\varphi^{(i)}}^{2}} > 0.
    \end{align*}
    This is equivalent to 
    \begin{align}\label{eq:pd}
    \P\parenthesis{\abs{s^{\top}\varphi^{(i)}} > 0} & = \P\parenthesis{\abs{s_{1}^{\top}x_{0}^{(i)} + s_{2}^{\top}u_{0}^{(i)} + s_{3}^{\top}\phi\parenthesis{x_{0}^{(i)}}} > 0} > 0.
    \end{align}
    We consider two cases of $ (s_{1}, s_{3}) $. If $ (s_{1}, s_{3}) \neq 0 $, by Assumption \ref{ass:initial_distr}, we have
    \begin{align}\label{eq:v1v3}
    \P\parenthesis{\abs{s_{1}^{\top}x_{0}^{(i)} + s_{3}^{\top}\phi\parenthesis{x_{0}^{(i)}}} > 0} > 0.
    \end{align} 
    Moreover, since $ x_{0}^{(i)} $ and $ u_{0}^{(i)} $ are independent, Eq.~\eqref{eq:pd} becomes
    \begin{align*}
        \P\parenthesis{\abs{s_{1}^{\top}x_{0}^{(i)} + s_{2}^{\top}u_{0}^{(i)} + s_{3}^{\top}\phi\parenthesis{x_{0}^{(i)}}} > 0} & \geq \P\parenthesis{\left. \abs{s_{1}^{\top}x_{0}^{(i)} + s_{2}^{\top}u_{0}^{(i)} + s_{3}^{\top}\phi\parenthesis{x_{0}^{(i)}}} > 0 \right\vert s_{2}^{\top}u_{0}^{(i)} = 0}\P(s_{2}^{\top}u_{0}^{(i)} = 0) \\ & = \P\parenthesis{\left. \abs{s_{1}^{\top}x_{0}^{(i)}  + s_{3}^{\top}\phi\parenthesis{x_{0}^{(i)}}} > 0 \right\vert s_{2}^{\top}u_{0}^{(i)} = 0}\P(s_{2}^{\top}u_{0}^{(i)} = 0) \\ & = \P\parenthesis{ \abs{s_{1}^{\top}x_{0}^{(i)}  + s_{3}^{\top}\phi\parenthesis{x_{0}^{(i)}}} > 0 }\parenthesis{1 - \P\parenthesis{\abs{s_{2}^{\top}u_{0}^{(i)}} > 0}}.
    \end{align*}
    By Eq.~\eqref{eq:v1v3} and the fact that $ \P\parenthesis{\abs{s_{2}^{\top}u_{0}^{(i)}} > 0} < 1 $ for all $ s_{2} $, Eq.~\eqref{eq:pd} holds when $ (s_{1}, s_{3}) \neq 0 $. Furthermore, if $ (s_{1}, s_{3}) = 0 $, Eq.~\eqref{eq:pd} simplifies to $ \P\parenthesis{\abs{s_{2}^{\top}u_{0}^{(i)}} > 0 } > 0 $ for all $ s_{2} \neq 0 $, which holds by definition of $ u_{0}^{(i)} $ in Algorithm \ref{alg:init estimation}. Therefore, we conclude that $ \Sigma $ is positive definite and thus invertible.
\end{proof}

\subsection{Proof of Theorem \ref{thm: landscape}}\label{sec: pf of local convexity}
We devote this subsection to the missing proofs of Theorem \ref{thm: landscape}. We begin by proving several auxiliary lemmas. 
Denote the value function and $ Q $ function conditioned on the initial position as 
\begin{eqnarray}
V_K(x) & = &  \mathbb{E}\left.\left[ \sum_{t=0}^{\infty}x_t^{\top} Q x_t + u_t^{\top}R u_t\right|x_0 = x,  u_t = - K_1 x_t - K_2 \phi(x_t)\right], \label{eq:cost-x0} \\ 
Q_K(x, u) & = & x^{\top}Qx + u^{\top}Ru + V_{K}(Ax + C\phi(x) + Bu). \label{eq:Q}
\end{eqnarray}

First, we provide a characterization of the value function below.
\begin{Lemma}[Value Function]\label{lemma:value_function}
The value function takes the form
\begin{eqnarray}
V_K(x) = x^{\top}P_{K_1} x + G_K(x),
\end{eqnarray}
where  $P_{K_1}$ satisfies Eq.~\eqref{eq:P} 
and $G_K(x)$ is defined as in Eq.~\eqref{eq:G_K}.
\end{Lemma}
\begin{proof}
By the Bellman equation, the value function  satisfies,
\begin{eqnarray}
V_K(x) &=&  x^{\top}Q x + \Big(K_1 x + K_2 \phi(x)\Big)^{\top}R \Big(K_1 x + K_2 \phi(x)\Big) \nonumber \\ & & \qquad\qquad\qquad\qquad\qquad\qquad\qquad\qquad + V_K\Big(Ax - B(K_1 x + K_2 \phi(x)) + C\phi(x)\Big)\nonumber\\
&=&  x^{\top}(Q + K_1^{\top} R K_1) x + \phi(x)^{\top} K_2^{\top} R K_2 \phi(x) + 2 x^{\top} K_1^{\top} RK_2 \phi(x)\nonumber \\
&&\qquad\qquad\qquad\qquad\qquad\qquad\qquad\qquad+ V_K\Big( (A - BK_1) x + (C - BK_2)\phi(x)\Big).\label{eq:Bellman}
\end{eqnarray}
Define $G_K(x) = V_K(x) - x^{\top}P_{K_1}x$. Replacing $ V_{K}(x) $ by $ G_{K}(x) + x^{\top}P_{K_{1}}x $ on both sides of Eq.~\eqref{eq:Bellman}, we have
\begin{eqnarray*}
&& x^{\top}P_{K_1}x + G_K(x) \\ &=&  x^{\top}(Q + K_1^{\top} R K_1) x + \phi(x)^{\top} K_2^{\top} R K_2 \phi(x) + 2 x^{\top} K_1^{\top} RK_2 \phi(x)\\
&&+ \Big( (A - B K_1) x + (C-B K_2 )\phi(x)\Big)^{\top}P_{K_1} \Big( (A - B K_1) x + (C-B K_2)\phi(x)\Big)  + G_K(x_1),
\end{eqnarray*}
with $x_1 = (A -B K_1) x + (C-BK_2 )\phi(x)$. Since $P_{K_1}$ satisfies \eqref{eq:P}, we have
\begin{eqnarray}
G_K(x) &=& \phi(x)^{\top} \Big( K_2^{\top} R K_2 + (C-BK_2)^{\top}P_{K_1}(C-BK_2) \Big) \phi(x)  \nonumber\\
&&+ 2 x^{\top} \Big(K_1^{\top} R K_2+ (A-BK_1)^{\top}P_{K_1}(C-BK_2) \Big)  \phi(x) + G_K(x_1) \nonumber \\
&=& \Tr \Big( \Big( K_2^{\top} R K_2 + (C-BK_2)^{\top}P_{K_1}(C-BK_2) \Big) \phi(x)\phi(x)^{\top}\Big) \nonumber \\
&& + 2 \Tr \Big (\Big(K_1^{\top} R K_2+ (A-BK_1)^{\top}P_{K_1}(C-BK_2) \Big)\phi(x)x^{\top} \Big)+ G_K(x_1) \label{eq:tr} \\
&=&   \Tr \Big( \Big( K_2^{\top} R K_2 + (C-BK_2)^{\top}P_{K_1}(C-BK_2) \Big)\sum_{t=0}^{\infty} \phi(x_t)\phi(x_t)^{\top}\Big)\nonumber \\
&& + 2 \Tr \Big (\Big(K_1^{\top} R K_2+ (A-BK_1)^{\top}P_{K_1}(C-BK_2) \Big)\sum_{t=0}^{\infty}\phi(x_t)x_t^{\top} \Big).\label{eq:unrolling G}
\end{eqnarray}
Here, we have used the matrix trace property in Eq.~\eqref{eq:tr}. By unrolling the recursive relation \eqref{eq:tr}, we obtain Eq.~\eqref{eq:unrolling G}. Therefore, Eq.~\eqref{eq:G_K} holds.
\end{proof}

Recall that we have defined the coefficient matrix $ H = (A, C) $, and the feature map $ \psi(x) = (x^{\top}, \phi(x)^{\top})^{\top} $. Then the dynamics \eqref{eq:dynamics} can be written as
\begin{align}
x_{t+1} = (A - BK_{1})x_{t} + (C - BK_{2})\phi(x_{t}) = (H - BK)\psi(x_{t}).
\end{align}
Since $ \phi(x) $ is $ \ell $-Lipschitz by Assumption \ref{ass:feature}, we know that $ \psi(x) $ is also $ \ell_{\psi} $-Lipschitz with $ \ell_{\psi} := \sqrt{1 + \ell^{2}} $. The following lemma gives us the gradient of the cost function $\mathcal{C}(K)$.

\begin{Lemma}[Gradient of $ \mcal{C}(K) $]\label{lemma: grad of cost function}
    The gradient of $ \mcal{C}(K) $ satisfies
    \begin{align}
    \nabla_{K}\mcal{C}(K) & = 2E_{K}\Sigma_{K}^{\psi\psi} - B^{\top}\Sigma_{K}^{G\psi},
    \end{align}
    where $ E_{K}$, $ \Sigma_{K}^{\psi\psi} $ and $ \Sigma_{K}^{G\psi} $ are defined as in Eq.~\eqref{eq:def-Ek}.
\end{Lemma}

\begin{proof}
  Recall the Bellman equation
  \begin{align}\label{eq:bellman-M}
  V_{K}(x) = x^{\top}Qx + (K\psi(x))^{\top}RK\psi(x) + V_{K}((H - BK)\psi(x)).
  \end{align}
  Taking gradient in $ K $ on both sides of Eq.~\eqref{eq:bellman-M}, we have
  \begin{align}\label{eq:grad V}
  \nabla_{K}V_{K}(x) & = 2RK\psi(x)\psi(x)^{\top} + \nabla_{K}V(x_{1}) + \parenthesis{\frac{\partial x_{1}}{\partial K}}^{\top}\nabla_{x}V_{K}(x_{1}),
  \end{align}
  where $ \nabla_{K}V(x_{1}) = \left.\frac{\partial V_{K}(x_{1})}{\partial K}\right\vert_{x_{1} = (H - BK)\psi(x)} $. Note the directional derivative of $ x_{1} $ in $ K $ along the direction  $ \Delta $ is $ x_{1}'\bracket{\Delta} = -B\Delta \psi(x) $. Since $ \nabla_{x}V_{K}(x) = 2P_{K_{1}}x + \nabla G(x) $, we have
  \begin{align*}
  x_{1}'\bracket{\Delta}^{\top}\nabla_{x}V_{K}(x_{1}) & = -\psi(x)^{\top}\Delta^{\top}B^{\top}\parenthesis{2P_{K_{1}}x_{1} + \nabla G_{K}(x_{1})} \\ & = \Tr \parenthesis{\Delta^{\top}\parenthesis{-2B^{\top}P_{K_{1}}x_{1}\psi(x)^{\top} - B^{\top}\nabla G_{K}(x_{1})\psi(x)^{\top}}}.
  \end{align*}
  Since $ x_{1} = (H - BK)\psi(x) $, it follows that
  \begin{align}
    \parenthesis{\frac{\partial x_{1}}{\partial K}}^{\top}\nabla_{x}V_{K}(x_{1}) & = -2B^{\top}P_{K_{1}}x_{1}\psi(x)^{\top} - B^{\top}\nabla G_{K}(x_{1})\psi(x)^{\top} \nonumber \\ & = - 2B^{\top}P_{K_{1}}(H- BK)\psi(x)\psi(x)^{\top} - B^{\top}\nabla G_{K}(x_{1})\psi(x)^{\top}. \label{eq:direction}
  \end{align}
  Substituting Eq.~\eqref{eq:direction} back into Eq.~\eqref{eq:grad V}, we obtain
  \begin{align*}
  \nabla_{K}V_{K}(x) & = 2RK\psi(x)\psi(x)^{\top}  - 2B^{\top}P_{K_{1}}(H- BK)\psi(x)\psi(x)^{\top} - B^{\top}\nabla G_{K}(x_{1})\psi(x)^{\top} + \nabla_{K}V(x_{1}) \\ & = \parenthesis{2RK - 2B^{\top}P_{K_{1}}(H- BK)}\psi(x)\psi(x)^{\top} - B^{\top}\nabla G_{K}(x_{1})\psi(x)^{\top} + \nabla_{K}V(x_{1}).
  \end{align*}
  Unrolling this recursive relation and apply the definition of $ E_{K} $ in \eqref{eq:def-Ek}, we conclude that
  \begin{align*}
    \nabla_{K}V_{K}(x) = 2E_{K}\sum_{t = 0}^{\infty}\psi(x_{t})\psi(x_{t})^{\top}  - B^{\top}\sum_{t = 0}^{\infty}\nabla G_{K}(x_{t+1})\psi(x_{t})^{\top}.
  \end{align*}
  Take expectation w.r.t. $ x_{0} = x $ and then we finish the proof. 
  \end{proof}

  With Lemma \ref{lemma:value_function} and Lemma \ref{lemma: grad of cost function}, we provide a formula for $\mcal{C}(K') - \mcal{C}(K)$ in the following lemma.

  \begin{Lemma}[Cost Difference Lemma]\label{lemma: cost difference}
      For $ K = (K_{1}, K_{2}) $ and $ K' = (K_{1}', K_{2}') $, we have
      \begin{align*}
      \mcal{C}(K') -\mcal{C}(K) & = \Tr\parenthesis{(K' - K)^{\top}(R + B^{\top}P_{K_{1}}B)(K' - K)\Sigma_{K'}^{\psi\psi}} + 2\Tr\parenthesis{(K' - K)^{\top}E_{K}\Sigma_{K'}^{\psi\psi}} \nonumber \\ & \qquad +  \E\bracket{\sum_{t = 0}^{\infty}\bracket{G_{K}((H - BK')\psi(x_{t}')) - G_{K}((H - BK)\psi(x_{t}'))}}.
      \end{align*}
  \end{Lemma}

  \begin{proof}
    By \cite[Lemma 10]{fazel2018global}, we have
    \begin{align}\label{eq:cost diff A}
    V_{K'}(x) - V_{K}(x) = \sum_{t = 0}^{\infty}A_{K}(x_{t}', u_{t}'),
    \end{align}
    where $ \curly{x_{t}'} $ is the trajectory generated by $ x_{0}' = x $ and $ u_{t}' = - K'\psi(x_{t}') $, and $ A_{K}(x, u) = Q_{K}(x, u) - V_{K}(x) $ is the advantage function. 
    
    For given $ u = -K'\psi(x) $, by definition \eqref{eq:cost-x0} and \eqref{eq:Q}, we have
    \begin{align}
    A_{K}(x, u) & = Q_{K}(x, u) - V_{K}(x) \nonumber \\ & = x^{\top}Qx + (K'\psi(x))^{\top}R(K'\psi(x)) + V_{K}((H - BK')\psi(x)) -  V_{K}(x) \nonumber\\ & = (K'\psi(x))^{\top}R(K'\psi(x)) - (K\psi(x))^{\top}R(K\psi(x)) \nonumber\\ & \qquad + V_{K}((H - BK')\psi(x)) - V_{K}((H - BK)\psi(x)) \nonumber\\ & = \psi(x)^{\top}(K' - K)^{\top}R(K' - K)\psi(x) + 2\psi(x)^{\top}(K' - K)^{\top}RK\psi(x) \nonumber\\ & \qquad + V_{K}((H - BK')\psi(x)) - V_{K}((H - BK)\psi(x)). \label{eq:diff V}
    \end{align}
    We next compute the last two terms in Eq.~\eqref{eq:diff V}. By Lemma \ref{lemma:value_function}, we notice
    \begin{align*}
        & V_{K}((H - BK')\psi(x))- V_{K}((H - BK)\psi(x))  \\ &  = \parenthesis{(H - BK')\psi(x)}^{\top}P_{K_{1}}\parenthesis{(H - BK')\psi(x)}- \parenthesis{(H - BK)\psi(x)}^{\top}P_{K_{1}}\parenthesis{(H - BK)\psi(x)} \\ & \qquad + G_{K}\parenthesis{(H - BK')\psi(x)}  - G_{K}\parenthesis{(H - BK)\psi(x)} \\ & = \psi(x)^{\top}(K' - K)^{\top}B^{\top}P_{K_{1}}B(K' - K)\psi(x) + 2\psi(x)^{\top}(K - K')^{\top}B^{\top}P_{K_{1}}(H - BK)\psi(x)  \\ & \qquad + G_{K}\parenthesis{(H - BK')\psi(x)}  - G_{K}\parenthesis{(H - BK)\psi(x)}.
    \end{align*}
    Substitution it back into Eq.~\eqref{eq:diff V}, we obtain
    \begin{align*}
    & A_{K}(x, u)  = \psi(x)^{\top}(K' - K)^{\top}(R + B^{\top}P_{K_{1}}B)(K' - K)\psi(x)\\ & \qquad  + 2\psi(x)^{\top}(K' - K)^{\top}(RK - B^{\top}P_{K_{1}}(H - BK))\psi(x) \\ & \qquad + G_{K}\parenthesis{(H - BK')\psi(x)}  - G_{K}\parenthesis{(H - BK)\psi(x)}.
    \end{align*}
    Finally, we take expectation of both sides of Eq.~\eqref{eq:cost diff A} w.r.t. $ x_{0} $, yielding
    \begin{align*}
    \mcal{C}(K') - \mcal{C}(K) & = \E\bracket{\sum_{t = 0}^{\infty}A_{K}(x_{t}', u_{t}')} \\ & = \Tr\parenthesis{(K' - K)^{\top}(R + B^{\top}P_{K_{1}}B)(K' - K)\E\bracket{\sum_{t = 0}^{\infty}\psi(x_{t}')\psi(x_{t}')^{\top}}} \\ & \qquad + 2\Tr\parenthesis{(K' - K)^{\top}(RK - B^{\top}P_{K_{1}}(H - BK))\E\bracket{\sum_{t = 0}^{\infty}\psi(x_{t}')\psi(x_{t}')^{\top}}} \\ & \qquad +  \E\bracket{\sum_{t = 0}^{\infty}\bracket{G_{K}((H - BK')\psi(x_{t}')) - G_{K}((H - BK)\psi(x_{t}'))}}.
    \end{align*}

\end{proof}

Next, we show that the state trajectory has an exponential decay property regardless of the initial state. In consequence, the cost function $\mcal{C}(\cdot)$ is bounded. 

\begin{Lemma}[Stability of the Trajectory $ \curly{x_{t}} $]\label{lemma:trajectory_stability}
    Assume Assumption \ref{ass:feature} holds,  $ K \in \Omega $ and $ \ell \leq \frac{1 - \rho_{1}}{4c_{1}c_{2}} $. The following results hold for each $t \geq 0 $:
    \begin{enumerate}[(a)]
        \item\label{item:stability a} For any $ x_{0} \in \R^{n} $, we have $ \norm{x_{t}} \leq c\rho^{t}\norm{x_{0}} $, where $ c = 2c_{1} $ and $ \rho = \frac{\rho_{1} + 1}{2} $. 
        \item Let $ \curly{x_{t}} $ and $ \curly{x_{t}'} $ be the state trajectories starting from $ x_{0} $ and $ x_{0}' $, respectively. Then $ \norm{x_{t} - x_{t}'} \leq c\rho^{t}\norm{x_{0} - x_{0}'} $, and consequently, $ \norm{\frac{\partial x_{t}}{\partial x_{0}}} \leq c\rho^{t} $. 

        \item Let $ \curly{x_{t}} $ and $ \curly{x_{t}'} $ be trajectories defined as above. Then $ \norm{\frac{\partial x_{t}}{\partial x_{0}} - \frac{\partial x_{t}'}{\partial x_{0}'}} \leq \frac{c_{2}\ell'c^{3}}{1 - \rho}\rho^{t - 1}\norm{x_{0} - x_{0}'} $. 
    \end{enumerate}
\end{Lemma}

\begin{proof}
  Let $ f(x) = (C - BK_{2})\phi(x) : \R^n \to \R^n $. Then, the dynamics \eqref{eq:dynamics} become $ x_{t+1} = (A - BK_{1})x_{t} + f(x_{t}) $. Also, by definition of $ \Omega $ in \eqref{eq:omega}, we have
  \begin{align*}
  \norm{f(x) - f(x')} & = \norm{(C - BK_{2})(\phi(x) - \phi(x'))} \leq \norm{C - BK_{2}}\norm{\phi(x) - \phi(x')} \leq c_{2}\ell\norm{x - x'},\\ 
  \norm{\nabla f(x) - \nabla f(x')} & = \norm{(C - BK_{2})(\nabla \phi(x) - \nabla \phi(x'))} = \norm{C - BK_{2}}\norm{\nabla \phi(x) - \nabla \phi(x')} \leq c_{2}\ell'\norm{x - x'}.
  \end{align*}
   Apply \cite[Lemma 4]{Qu2020CombiningMA} and then we finish the proof. 
\end{proof}

We provide an upper bound on $ \norm{P_{K_{1}}} $ that will be used in the rest of this subsection. 
\begin{Lemma}\label{lemma: bound P_K}
    Assume Assumption \ref{ass: cost function} holds. If $ \norm{K - K^{\rm lin}}_F \leq \delta \leq 1$, we have
    \begin{align}\label{eq:C_P}
    \norm{P_{K_{1}}} \leq C_{P} \coloneqq  \frac{2c_{1}^{2}(1 + \Gamma)^{2}}{1 - \rho_{1}} ,
    \end{align}
    where $ P_{K_{1}} $ is the solution to the Lyapunov equation \eqref{eq:P}.
\end{Lemma}

\begin{proof}
  By unrolling the Lyapunov equation \eqref{eq:P}, we have
  \begin{align*}
  P_{K_{1}} = \sum_{t=0}^{\infty}\parenthesis{(A - BK_{1})^{\top}}^{t}(Q + K_{1}^{\top}RK_{1})(A-BK_{1})^{t}.
  \end{align*}    
  Since $ \norm{(A - BK_{1})^{t}} \leq c_{1}\rho_{1}^{t} $ for some $ c_{1} > 1 $ and $ \rho_{1} \in (0, 1) $, it follows
  \begin{align}\label{eq:bound P}
  \norm{P_{K_{1}}} \leq  \sum_{t = 0}^{\infty}c_{1}^{2}\rho^{2t}\norm{Q + K_{1}^{\top}RK_{1}} \leq \frac{c_{1}^{2}}{1 - \rho_{1}^{2}}\norm{Q + K_{1}^{\top}RK_{1}} \leq \frac{c_{1}^{2}}{1 - \rho_{1}}\norm{Q + K_{1}^{\top}RK_{1}}.
  \end{align}
  Moreover, by Assumption \ref{ass: cost function}, we have $ \norm{Q}, \norm{R} \leq 1 $, leading to
  \begin{align}\label{eq:Q +KRK}
  \norm{Q + K_{1}^{\top}RK_{1}} \leq 1 + \norm{K_{1}}^{2} \leq 1 + \norm{K}^{2}.
  \end{align}
  Also, since $ \norm{K^{\rm lin}}_{F} \leq \Gamma $ and $ \norm{K - K^{\rm lin}}_{F} \leq \delta \leq 1 $, we have
\begin{align}\label{eq:bound K}
    \norm{K} \leq \norm{K - K^{\rm lin}} + \norm{K^{\rm lin}} \leq \norm{K - K^{\rm lin}}_{F} + \norm{K^{\rm lin}}_{F} \leq 1 + \Gamma.
\end{align}
    Finally, combining Eq.~\eqref{eq:Q +KRK} and \eqref{eq:bound K}, Eq.~\eqref{eq:bound P} becomes
  \begin{align*}
  \norm{P_{K_{1}}} \leq  \frac{c_{1}^{2}(1 + (1 + \Gamma)^{2})}{1 - \rho_{1}} < \frac{2c_{1}^{2}(1 + \Gamma)^{2}}{1 - \rho_{1}} \eqqcolon C_{P},
  \end{align*}
  which finishes the proof.
  \end{proof}

The key property to guarantee the local strong convexity of the cost function $ \mcal{C}(\cdot) $ is the local Lipschitz continuity of $\nabla G_K(x)$. Recall $c = 2c_1$ and $\rho = (\rho_1 + 1)/2$.
\begin{Lemma}[Local Lipschitz Continuity of $ \nabla G_{K}(x) $]\label{lemma:lip of G}
Assume Assumptions \ref{ass:feature}, \ref{ass: cost function} and \ref{ass:initial_distr} hold.
    When $ \norm{K - K^{\rm lin}}_F \leq \delta $ and  $ \norm{x}, \norm{x'} \leq (c_{1} + c_{2})cD_{0} $, we have
    \begin{align}
    \norm{\nabla G_{K}(x) - \nabla G_{K}(x')} \leq L \norm{x - x'}, \label{eq:nabla G diff}
    \end{align}
    where $ L $ is defined as in Eq.~\eqref{eq:L}.
\end{Lemma}

\begin{proof}
    Define  $ \pi_{K}(x_{t}) = -K_{1}x_{t} - K_{2}\phi(x_{t})  $. Also, let
  \begin{align*}
  F_{K}^{12} & = \parenthesis{F_{K}^{21}}^{\top} = K_{1}^{\top}RK_{2} + (A - BK_{1})^{\top}P_{K_{1}}(C - BK_{2}), \quad \text{and} \quad \\ 
  F_{K}^{22} & = K_{2}^{\top}RK_{2} + (C - BK_{2})^{\top}P_{K_{1}}(C - BK_{2}). 
  \end{align*} 
    By the definition of $ G_{K}(x) $ in Eq.~\eqref{eq:G_K}, we first compute its gradient  as follows
  \begin{align*}
  \bracket{\nabla G_{K}(x)}^{\top} & = 2\sum_{t = 0}^{\infty}\bracket{\phi(x_{t})^{\top}(F_{K}^{12})^{\top} + x_{t}^{\top}F_{K}^{12}\frac{\partial \phi(x_{t})}{\partial x_{t}}}\frac{\partial x_{t}}{\partial x} + 2\sum_{t = 0}^{\infty}\phi(x_{t})^{\top}F_{K}^{22}\frac{\partial \phi(x_{t})}{\partial x_{t}}\frac{\partial x_{t}}{\partial x} \\ & = 2\sum_{t = 0}^{\infty}\bracket{\phi(x_{t})^{\top}F_{K}^{21} - \pi_{K}(x_{t})^{\top}RK_{2}\frac{\partial \phi(x_{t})}{\partial x_{t}} + x_{t+1}^{\top}P_{K_{1}}(C - BK_{2})\frac{\partial \phi(x_{t})}{\partial x_{t}}}\frac{\partial x_{t}}{\partial x},
  \end{align*}
  As such, for two states $ x $ and $ x' $, we have
  \begin{align}
  & \norm{\nabla G_{K}(x) - \nabla G_{K}(x')} \nonumber\\ & \leq 2\sum_{t = 0}^{\infty}\left\| \bracket{\phi(x_{t}) - \phi(x_{t}')}^{\top}F_{K}^{21} - \bracket{\pi_{K}(x_{t})^{\top}RK_{2}\frac{\partial \phi(x_{t})}{\partial x_{t}} - \pi_{K}(x_{t}')^{\top}RK_{2}\frac{\partial \phi(x_{t}')}{\partial x_{t}'}} \right. \nonumber\\ & \qquad\qquad \left. + x_{t+1}^{\top}P_{K_{1}}(C - BK_{2})\frac{\partial \phi(x_{t})}{\partial x_{t}} - (x_{t+1}')^{\top}P_{K_{1}}(C - BK_{2})\frac{\partial \phi(x_{t}')}{\partial x_{t}'}\right\|\norm{\frac{\partial x_{t}}{\partial x}} \nonumber \\ & \qquad + 2\sum_{t=0}^{\infty}\norm{\phi(x_{t}')^{\top}F_{K}^{21} - \pi_{K}(x_{t}')^{\top}RK_{2}\frac{\partial \phi(x_{t}')}{\partial x_{t}'} +  (x_{t+1}')^{\top}P_{K_{1}}(C- BK_{2})\frac{\partial \phi(x_{t}')}{\partial x_{t}'}}\norm{\frac{\partial x_{t}}{\partial x} - \frac{\partial x_{t}'}{\partial x'}}. \label{eq: nabla G diff-1}
  \end{align} 
  
  We compute the bounds one by one. Firstly, since $ \norm{x_{t} - x_{t}'} \leq c\norm{x - x'} $ and $ \phi $ is $ \ell $-Lipschitz, we have
  \begin{align}\label{eq:phi-F-1}
  \norm{\bracket{\phi(x_{t}) - \phi(x_{t}')}^{\top}F_{K}^{21}} \leq \ell \norm{x_{t} - x_{t}'}\norm{F_{K}^{21}} \leq \ell c\norm{x - x'}\norm{F_{K}^{21}}.
  \end{align}
  Thus, it suffices to establish a bound on $ \norm{F_{K}^{21}} $. Note
  \begin{align}
  \norm{F_{K}^{21}} & = \norm{K_{2}^{\top}RK_{1} + (C - BK_{2})^{\top}P_{K_{1}}(A - BK_{1})} \nonumber\\ & \leq \norm{K_{1}}\norm{R}\norm{K_{2}} + \norm{C - BK_{2}}\norm{P_{K_{1}}}\norm{A - BK_{1}}. \label{eq: F_K^21-1}
  \end{align}
  Since $ K \in \Lambda(\delta) \subset \Omega $, we have $ \norm{A - BK_{1}} \leq c_{1} $ and $ \norm{C - BK_{2}} \leq c_{2} $. Also, by Lemma \ref{lemma: bound P_K}, we have $ \norm{P_{K_{1}}} \leq C_{P} $. Thus, Eq.~\eqref{eq: F_K^21-1} becomes
  \begin{align*}
    \norm{F_{K}^{21}} & \leq \norm{K_{1}}_{F}\norm{K_{2}}_{F} + c_{1}c_{2}C_{P} \tag{$ \norm{R} \leq 1$ } \\ & \leq \frac{1}{2}\norm{K}_{F}^{2} + c_{1}c_{2}C_{P} \tag{using $ ab \leq \frac{1}{2}(a^{2} + b^{2}) $}\\ & \leq \frac{1}{2}(1 + \Gamma)^{2} + c_{1}c_{2}C_{P} \tag{$ \norm{K}_{F} \leq 1 + \Gamma $}\\ & \leq \frac{5c_{2}c_{1}^{3}(1 + \Gamma)^{2}}{2(1 - \rho_{1})} \eqqcolon C_{F}^{21}, \tag{substituting $ C_{P} $ from Lemma \ref{lemma: bound P_K}}
  \end{align*}
  It follows from Eq.~\eqref{eq:phi-F-1} and the fact $ c = 2c_{1} $ that
  \begin{align}\label{eq:phi-F-2}
      \norm{\bracket{\phi(x_{t}) - \phi(x_{t}')}^{\top}F_{K}^{21}} \leq  \frac{5\ell c_{2}c_{1}^{4}(1 + \Gamma)^{2}}{1 - \rho_{1}}\norm{x - x'}.
  \end{align}
  
  Next, note that $ \norm{\pi_{K}(x)} \leq (\norm{K_{1}} + \ell\norm{K_{2}})\norm{x} $ and $ \norm{\pi_{K}(x) - \pi_{K}(x')} \leq \parenthesis{\norm{K_{1}} + \ell\norm{K_{2}}}\norm{x - x'} $ for any $ x $ and $ x' $. Then, it follow from Lemma \ref{lemma:trajectory_stability} that
  \begin{align}
  & \norm{\pi_{K}^{\top}(x_{t})RK_{2}\frac{\partial \phi(x_{t})}{\partial x_{t}} - \pi_{K}(x_{t}')^{\top}RK_{2}\frac{\partial \phi(x_{t}')}{\partial x_{t}'}} \nonumber \\ & \leq \norm{\parenthesis{\pi_{K}(x_{t}) - \pi_{K}(x_{t}')}^{\top}R K_{2}\frac{\partial \phi(x_{t})}{\partial x_{t}}}  + \norm{\pi_{K}^{\top}(x_{t}')R K_{2} \parenthesis{\frac{\partial \phi(x_{t})}{\partial x_{t}} - \frac{\partial \phi(x_{t}')}{\partial x_{t}'}}}\nonumber \\  & \leq \norm{\pi_{K}(x_{t}) - \pi_{K}(x_{t}')}\norm{R}\norm{K_{2}}\norm{\frac{\partial \phi(x_{t})}{\partial x_{t}}}  + \norm{\pi_{K}(x_{t}')}\norm{R}\norm{K_{2}}\norm{\frac{\partial \phi(x_{t})}{\partial x_{t}} - \frac{\partial \phi(x_{t}')}{\partial x_{t}'}}\nonumber \\ &  \leq \ell(\norm{K_{1}} + \ell\norm{K_{2}})\norm{K_{2}}\norm{x_{t} - x_{t}'}  + \ell'(\norm{K_{1}} + \ell\norm{K_{2}})\norm{x_{t}}\norm{K_{2}}\norm{x_{t} - x_{t}'}. \label{eq:pi-R-K}
  \end{align} 
  Furthermore, let $ D $ be such that  $ \norm{x_{t}} \leq c\norm{x_{0}} \leq (c_{1}+c_{2})c^{2}D_{0} \eqqcolon D $. Since $ \norm{K_{1}}\norm{K_{2}} \leq \frac{1}{2}\norm{K}_{F}^{2} $ and by Lemma \ref{lemma: estimation C(K)} that $ \norm{x_{t} - x_{t}'} \leq c\norm{x - x'} $, Eq.~\eqref{eq:pi-R-K} becomes 
  \begin{align*}
    & \norm{\pi_{K}^{\top}(x_{t})RK_{2}\frac{\partial \phi(x_{t})}{\partial x_{t}} - \pi_{K}(x_{t}')^{\top}RK_{2}\frac{\partial \phi(x_{t}')}{\partial x_{t}'}}
    \\ & \leq (\ell/2 + \ell^{2})\norm{K}_{F}^{2}c\norm{x - x'} + (\ell'/2 + \ell\ell')\norm{K}_{F}^{2}D c\norm{x - x'} \\ & \leq (\ell/2 + \ell^{2} + D\ell'/2 + D\ell\ell')(1 + \Gamma)^{2}c\norm{x - x'} \tag{$ \norm{K}_{F} \leq 1 + \Gamma $} \\ & = (1/2 + \ell)(\ell + D\ell')(1 + \Gamma)^{2}c\norm{x - x'} \\ & \leq 2(\ell + D\ell')(1 + \Gamma)^{2}c_{1}\norm{x - x'}. \tag{using $ \ell \leq 1/2 $ and $ c = 2c_{1} $}
  \end{align*}

  Moreover, note that
  \begin{align}
  & \norm{x_{t+1}^{\top}P_{K_{1}}(C - BK_{2})\frac{\partial \phi(x_{t})}{\partial x_{t}} - (x_{t+1}')^{\top}P_{K_{1}}(C - BK_{2})\frac{\partial \phi(x_{t}')}{\partial x_{t}'}}\nonumber \\ & \leq \norm{(x_{t+1}- x_{t+1}')^{\top}P_{K_{1}}(C - BK_{2})\frac{\partial \phi(x_{t})}{\partial x_{t}}} + \norm{(x_{t+1}')^{\top}P_{K_{1}}(C -BK_{2})\parenthesis{\frac{\partial \phi(x_{t})}{\partial x_{t}} - \frac{\partial \phi(x_{t}')}{\partial x_{t}'}}} \nonumber\\ & \leq \norm{x_{t+1} - x_{t+1}'}\norm{P_{K_{1}}}\norm{C - BK_{2}}\norm{\frac{\partial \phi(x_{t})}{\partial x_{t}}} + \norm{x_{t+1}'}\norm{P_{K_{1}}}\norm{C - BK_{2}}\norm{\frac{\partial \phi(x_{t})}{\partial x_{t}} - \frac{\partial \phi(x_{t}')}{\partial x_{t}'}}. \label{eq:xt+1}
  \end{align}
  Recall that $ \phi $ is $ \ell $-Lipschitz and $ \ell' $-gradient-Lipschitz. Also, we have $ \norm{P_{K_{1}}} \leq C_{P} $ and $ \norm{C - BK_{2}} \leq c_{2} $. Based on these facts, by applying Lemma \ref{lemma:trajectory_stability}, Eq.~\eqref{eq:xt+1} can be bounded as
  \begin{align*}
    & \norm{x_{t+1}^{\top}P_{K_{1}}(C - BK_{2})\frac{\partial \phi(x_{t})}{\partial x_{t}} - (x_{t+1}')^{\top}P_{K_{1}}(C - BK_{2})\frac{\partial \phi(x_{t}')}{\partial x_{t}'}}
     \\ & \leq C_{P}c_{2}\ell\norm{x_{t+1} - x_{t+1}'} + DC_{P}c_{2}\ell'\norm{x_{t} - x_{t}'} \tag{using $ \norm{x_{t+1}'} \leq D $}\\ & \leq cc_{2}C_{P}(\ell + D\ell')\norm{x - x'} \tag{using Lemma \ref{lemma:trajectory_stability}} \\ & \leq \frac{3c_{2}c_{1}^{3}(1 + \Gamma)^{2}}{1 - \rho_{1}}(\ell + D \ell')\norm{x - x'}. \tag{substituting $ C_{P} $ from Lemma \ref{lemma: bound P_K}}
  \end{align*}

  Finally, by Lipschitz property of $ \phi $ and $ \pi_{K} $, we obtain
  \begin{align*}
      & \norm{\phi(x_{t}')^{\top}F_{K}^{21} - \pi_{K}(x_{t}')^{\top}RK_{2}\frac{\partial \phi(x_{t}')}{\partial x_{t}'} +  (x_{t+1}')^{\top}P_{K_{1}}(C- BK_{2})\frac{\partial \phi(x_{t}')}{\partial x_{t}'}} \\ & \leq \ell\norm{x_{t}'} C_{F}^{21} + \ell(\norm{K_{1}} + \ell \norm{K_{2}})\norm{x_{t}'}\norm{K_{2}} + C_{P}c_{2}\ell\norm{x_{t+1}'} \tag{$ \norm{F_{K}^{21} }\leq C_{F}^{21} $ and $ \norm{P_{K_{1}}} \leq C_{P} $}\\ & \leq \ell D C_{F}^{21} + (\ell/2+ \ell^{2})\norm{K}_{F}^{2}D + c_{2}\ell C_{P}D \tag{using $ \norm{x_{t}'} \leq D $} \\ & \leq \ell D C_{F}^{21} + (3/2)\ell(1 + \Gamma)^{2}D + c_{2}\ell C_{P}D \tag{using $ \ell \leq 1 $} \\ & \leq \frac{11\ell D c_{2}c_{1}^{3}(1 + \Gamma)^{2}}{2(1 - \rho_{1})}. \tag{substituting $ C_{P} $ from Lemma \ref{lemma: bound P_K}}
  \end{align*}
  
  Plugging all these results into Eq.~\eqref{eq:diff nabla G}, and using the facts that $ \norm{\frac{\partial x_{t}}{\partial x}} \leq c\rho^{t} $ and $ \norm{\frac{\partial x_{t}}{\partial x} - \frac{\partial x_{t}'}{\partial x'}} \leq \frac{c_{2}\ell'c^{3}}{1 - \rho}\rho^{t-1}\norm{x - x'} $, we conclude that 
  \begin{align*}
   & \norm{\nabla G_{K}(x) - \nabla G_{K}(x')} \\ &  \leq 2\sum_{t = 0}^{\infty}\bracket{\frac{5\ell c_{2}c_{1}^{4}(1 + \Gamma)^{2}}{1 - \rho_{1}} + 2(\ell + D\ell')(1 + \Gamma)^{2}c_{1} + \frac{3c_{2}c_{1}^{3}(1 + \Gamma)^{2}}{1 - \rho_{1}}(\ell + D \ell')}\norm{x - x'} \norm{\frac{\partial x_{t}}{\partial x}} \\ & \qquad + 2 \sum_{t = 1}^{\infty}\frac{11\ell D c_{2}c_{1}^{3}(1 + \Gamma)^{2}}{2(1 - \rho_{1})}\norm{\frac{\partial x_{t}}{\partial x} - \frac{\partial x_{t}'}{\partial x'}} \\ & \leq \frac{2c}{1 - \rho}\parenthesis{\frac{10c_{2}c_{1}^{4}(1 + \Gamma)^{2}}{1 - \rho_{1}}\ell + \frac{5c_{2}c_{1}^{3}(1 + \Gamma)^{2}D}{1 - \rho_{1}}\ell'}\norm{x - x'} \\ & \qquad  + 2\cdot\frac{11\ell D c_{2}c_{1}^{3}(1 + \Gamma)^{2}}{2(1 - \rho_{1})}\cdot \frac{c_{2}\ell'c^{3}}{(1 - \rho)^{2}}\norm{x - x'} \tag{using Lemma \ref{lemma:trajectory_stability}} \\ & \leq \parenthesis{\frac{40c_{2}c_{1}^{5}(1 + \Gamma)^{4}}{(1 - \rho_{1})^{2}}\ell + \frac{176Dc_{2}^{2}c_{1}^{6}(1 + \Gamma)^{2}}{(1 - \rho_{1})^{3}}\ell'}\norm{x - x'} \tag{using $ c = 2c_{1} $ and $ \rho = (\rho_{1}+2)/2 $}\\ & = \parenthesis{\frac{5c_{2}c^{5}(1 + \Gamma)^{4}}{16(1 - \rho)^{2}}\ell + \frac{3Dc_{2}^{2}c^{6}(1 + \Gamma)^{2}}{16(1 - \rho)^{3}}\ell'}\norm{x - x'},
  \end{align*}
  which shows that $ \nabla G_{K}(x) $ is $ L $-Lipschitz in $ x $. 
\end{proof}

The following result establishes a bound on the directional derivative of the state. 
\begin{Lemma}\label{lemma:directional derivative of state}
Assume Assumption \ref{ass:feature} holds. 
    The directional derivative of $ x_{t} $ w.r.t. $ K = (K_{1}, K_{2}) $ along the direction $ \Delta = (\Delta_{1}, \Delta_{2}) $ satisfies
    \begin{align}\label{eq:dir bound}
    \norm{x_{t}'[\Delta]} \leq \frac{\sqrt{2}c^{2}\Gamma}{ 1- \rho}\rho^{t}\norm{x_{0}}\norm{\Delta}.
    \end{align}
\end{Lemma}

\begin{proof}
Recall the dynamics are 
\begin{align*}
x_{t+1} = (A - BK_{1})x_{t} + (C - BK_{2})\phi(x_{t}).
\end{align*}
We compute the directional of $x_{t+1}$ derivative w.r.t $ K = (K_{1}, K_{2}) $ along the direction $ \Delta = (\Delta_{1},  \Delta_{2}) $:
\begin{align}
x_{t+1}'[\Delta] &  = (A - BK_{1})x_{t}'[\Delta] - B\Delta_{1}x_{t} + (C - BK_{2})\frac{\partial \phi(x_{t})}{\partial x_{t}}x_{t}'[\Delta] - B\Delta_{2}\phi(x_{t}) \nonumber \\ & = \sum_{k = 0}^{t}(A - BK_{1})^{t-k}\parenthesis{-B\Delta_{1}x_{k} + (C - BK_{2})\frac{\partial \phi(x_{k})}{\partial x_{k}}x_{k}'[\Delta] - B\Delta \phi(x_{k})}. \label{eq:dir xt}
\end{align}
Note that for $ K \in \Omega $, we have $ \norm{A - BK_{1}}^{t} \leq c_{1}\rho_{1}^{t} $ for each $ t \geq 0 $ and $ \norm{C - BK_{2}} \leq c_{2} $. Also, the Lipschitz property of $ \phi $ implies that $ \norm{\frac{\partial \phi(x)}{\partial x}} \leq \ell $ and $ \norm{\phi(x)} \leq \ell\norm{x} $. Hence, taking the norm of both sides of Eq.~\eqref{eq:dir xt} results in
\begin{align*}
\norm{x_{t+1}'[\Delta]} & \leq \sum_{k = 0}^{t}c_{1}\rho_{1}^{t-k}\parenthesis{\norm{B}\norm{\Delta_{1}}\norm{x_{k}} + c_{2}\ell \norm{x_{k}'[\Delta]} + \norm{B}\norm{\Delta_{2}}\ell \norm{x_{k}}} \\ & \leq \sum_{k = 0}^{t}c_{1}c_{2}\ell\rho_{1}^{t-k}\norm{x_{k}'[\Delta]} + \sum_{k = 0}^{t}c_{1}\rho_{1}^{t - k}\norm{B}(\norm{\Delta_{1}} +  \ell \norm{\Delta_{2}})\norm{x_{k}}.
\end{align*}
Since $ \norm{x_{k}} \leq c\rho^{k}\norm{x_{0}} $ by part (a) of Lemma \ref{lemma:trajectory_stability}, it follows
\begin{align*}
    \norm{x_{t+1}'[\Delta]}  & \leq \sum_{k = 0}^{t}c_{1}c_{2}\ell\rho_{1}^{t-k}\norm{x_{k}'[\Delta]} + \sum_{k = 0}^{t}c_{1}\rho_{1}^{t - k}\norm{B}(\norm{\Delta_{1}} +  \ell \norm{\Delta_{2}}) c\rho^{k}\norm{x_{0}} \\ & = \sum_{k = 0}^{t}c_{1}c_{2}\ell\rho_{1}^{t-k}\norm{x_{k}'[\Delta]} + c_{1}c\norm{B}(\norm{\Delta_{1}} +  \ell \norm{\Delta_{2}})\norm{x_{0}}\sum_{k = 0}^{t}(\rho/\rho_{1})^{k}
    \\ & = \sum_{k = 0}^{t}c_{1}c_{2}\ell\rho_{1}^{t-k}\norm{x_{k}'[\Delta]} + c_{1}c\norm{B}(\norm{\Delta_{1}} +  \ell \norm{\Delta_{2}})\norm{x_{0}} \frac{\rho^{t+1} - \rho_{1}^{t+1}}{\rho - \rho_{1}}.
\end{align*}
To prove Eq.~\eqref{eq:dir bound}, we assume that $ x_{t}'[\Delta] \leq \alpha \rho^{t} $,where $ \alpha = \frac{2c_{1}c\norm{x_{0}}\norm{B}(\norm{\Delta_{1}} + \ell\norm{\Delta_{2}})}{(\rho - \rho_{1})} $. Consequently, we have
\begin{align*}
\frac{\norm{x_{t+1}'[\Delta]}}{\alpha \rho^{t+1}} & \leq \sum_{k = 0}^{t}c_{1}c_{2}\ell \rho_{1}^{t-k} \frac{\alpha \rho^{k}}{\alpha \rho^{t+1}} + \frac{1}{2}(\rho^{t+1} - \rho_{1}^{t+1}) \\ & = \frac{c_{1}c_{2}\ell}{\rho}\sum_{k = 0}^{t}(\rho_{1}/\rho)^{k} +  \frac{1}{2}(\rho^{t+1} - \rho_{1}^{t+1}) \\ & \leq \frac{c_{1}c_{2}\ell}{\rho}\cdot\frac{(\rho_{1}/\rho)^{t+1} - 1}{\rho_{1}/\rho - 1} + \frac{1}{2}.
\end{align*}
Since $ \rho = (\rho_{1} + 1)/2 $, one can see that $ 0 < 1 - (\rho_{1}/\rho)^{t+1} < 1 $. Therefore, it holds from $ \ell \leq \frac{1 - \rho_{1}}{4c_{1}c_{2}} $ that
\begin{align*}
    \frac{\norm{x_{t+1}'[\Delta]}}{\alpha \rho^{t+1}}  \leq \frac{c_{1}c_{2}\ell}{\rho - \rho_{1}} + \frac{1}{2} \leq 1.
\end{align*}
By induction, we know that for each $ t \geq 1$, it holds
\begin{align*}
\norm{x_{t}'[\Delta]} \leq 2c_{1}c\norm{B}\norm{x_{0}}(\norm{\Delta_{1}} + \ell \norm{\Delta_{2}})\frac{\rho^{t}}{\rho - \rho_{1}}.
\end{align*}
Since $ \ell \leq 1 $, we have $ \norm{\Delta_{1}} + \ell\norm{\Delta_{2}} \leq \sqrt{2}\norm{\Delta} $. Additionally, by the definition of $ c $ and $ \rho $, we have $ \frac{2c_{1}c}{\rho - \rho_{1}} = \frac{c^{2}}{1 - \rho} $. Consequently, we conclude that
\begin{align*}
    \norm{x_{t}'[\Delta]} \leq \frac{\sqrt{2}c^{2}\Gamma}{ 1- \rho}\rho^{t}\norm{x_{0}}\norm{\Delta},
\end{align*}
where we have used the fact $ \norm{B} \leq \Gamma $.

\end{proof}

With Lemma \ref{lemma:directional derivative of state}, we establish the perturbation analysis of the covariance matrices $\Sigma_{K}^{\psi\psi}$ and $ \Sigma_{K}^{G\psi} $ and provide an upper bound on $ \norm{E_{K}} $. 

\begin{Lemma} \label{lemma: bound Sigma}
Assume Assumptions \ref{ass:feature}, \ref{ass: cost function} and \ref{ass:initial_distr} hold.
    For $ K, K' \in \Lambda(\delta) $, there exist constants $ C_{E}, C_{1} $ and $ C_{2} $ defined as in Eq.~\eqref{eq:CE} such that
    \begin{align}
    \norm{E_{K}} \leq C_{E}\norm{K - K^{\rm lin}}, \quad  \norm{\Sigma_{K'}^{\psi \psi} - \Sigma_{K}^{\psi\psi}}_{F}  \leq C_{1}\norm{K' - K}_{F}, \quad \text{and} \nonumber \\ \norm{\E\bracket{\sum_{t = 0}^{\infty}\nabla G_{K}(x_{t+1})(x_{t})^{\top}} - \E\bracket{\sum_{t = 0}^{\infty}\nabla G_{K}(x_{t+1}')(x_{t}')^{\top}}}_{F}  \leq C_{2}\norm{K' - K}_{F}. \label{eq:C1,C2,CE}
    \end{align}
\end{Lemma}

\begin{proof}[Proof of Lemma \ref{lemma: bound Sigma}]
    First, we prove $ \norm{\Sigma_{K'}^{\psi \psi} - \Sigma_{K}^{\psi\psi}}_{F} \leq C_{1}\norm{K' - K}_{F} $.
    Note that the directional derivative of $ \Sigma_{K}^{\psi\psi} $ w.r.t. $ K $ along the direction $ \Delta $ is
    \begin{align*}
        (\Sigma_{K}^{\psi\psi})'[\Delta] = \E\bracket{\sum_{t = 0}^{\infty}\parenthesis{\frac{\partial \psi(x_{t})}{\partial x_{t}}x_{t}'[\Delta]\psi(x_{t})^{\top} + \psi(x_{t})\frac{\partial \psi(x_{t})}{\partial x_{t}}x_{t}'[\Delta]^{\top}}}.
    \end{align*}
    Taking the norm of both sides, since $ \psi $ is $ \ell_{\psi} $-Lipschitz, we obtain
    \begin{align}
    \norm{(\Sigma_{K}^{\psi\psi})'[\Delta]}_{F} & \leq \E\bracket{\sum_{t = 0}^{\infty}2\ell_{\psi}\norm{x_{t}'[\Delta]}\norm{\psi(x_{t})}} \nonumber\\ & \leq \E\bracket{\sum_{t = 0}^{\infty}2\ell_{\psi}\frac{\sqrt{2}c^{2}\Gamma}{1 - \rho}\norm{x_{0}}\norm{\Delta}\ell_{\psi} c\rho^{t}\norm{x_{0}}} \label{eq:sigma+5.9}\\ &  \leq \frac{4c^{3}\Gamma D_{0}^{2}}{( 1- \rho)^{2}}\norm{\Delta}_{F}. \label{eq:Sigma-Delta}
    \end{align}
    Here, Eq.~\eqref{eq:sigma+5.9} is a consequence of Lemma \ref{lemma:trajectory_stability} and \ref{lemma:directional derivative of state}. Also, we have employed the facts that $ \ell_{\psi} \leq \sqrt{2}  $ and $ \E\bracket{\norm{x_{0}}} \leq D_{0} $ to derive Eq.~\eqref{eq:Sigma-Delta}. 
    Now, set $ g(t) = \Sigma_{K + t(K' - K)}^{\psi\psi} $. Then, its derivative in $t$ is $  g'(t) = \parenthesis{\Sigma_{K + t(K' - K)}^{\psi\psi}}'[K' - K] $.
     Since Eq.~\eqref{eq:Sigma-Delta} holds for arbitrary $ K $, we deduce
    \begin{align*}
    \norm{\Sigma_{K'}^{\psi\psi} - \Sigma_{K}^{\psi\psi}}_{F}  = \norm{\int_{0}^{1}g'(t){\rm d}t}_{F}\leq \int_{0}^{1}\norm{g'(t)}_{F}{\rm d}t \leq \frac{4c^{3}\Gamma D_{0}^{2}}{( 1- \rho)^{2}}\norm{K' - K}_{F}.
    \end{align*}

    Next, we show $  \norm{\E\bracket{\sum_{t = 0}^{\infty}\nabla G_{K}(x_{t+1})(x_{t})^{\top}} - \E\bracket{\sum_{t = 0}^{\infty}\nabla G_{K}(x_{t+1}')(x_{t}')^{\top}}}_{F} \leq C_{2}\norm{K' - K}_{F} $. Since $ \norm{x_{t}} \leq c\norm{x_{0}} \leq cD_{0} $, we apply Lemma \ref{lemma:lip of G} to obtain 
    \begin{align*}
    & \norm{\nabla G_{K}(x_{t+1})(x_{t})^{\top} - \nabla G_{K}(x_{t+1}')(x_{t}')^{\top}}_{F} \\ & \leq \norm{\nabla G_{K}(x_{t+1}) - \nabla G_{K}(x_{t+1}')}\norm{x_{t}} + \norm{\nabla G_{K}(x_{t+1}')}\norm{x_{t} - x_{t}'}   \\ & \leq L\norm{x_{t+1} - x_{t+1}'}\norm {x_{t}} + L\norm{x_{t+1}'}\norm{x_{t} - x_{t}'} \tag{using Lemma \ref{lemma:lip of G}}
    \end{align*}
    Moreover, as a consequence of Lemma \ref{lemma:directional derivative of state}, we have $ \norm{x_{t} - x_{t'}} \leq \frac{\sqrt{2}c^{2}\Gamma}{1 - \rho}\rho^{t}\norm{x_{0}}\norm{K' - K} $ for any $ t \geq 1 $. Thus, together with the fact $ \norm{x_{t}} \leq c\rho^{t}\norm{x_{0}} $, we have
    \begin{align}
        & \norm{\nabla G_{K}(x_{t+1})(x_{t})^{\top} - \nabla G_{K}(x_{t+1}')(x_{t}')^{\top}}_{F} \nonumber \\ & \leq L \frac{\sqrt{2}c^{2}\Gamma}{1 - \rho}\rho^{t+1}c\rho^{t}\norm{x_{0}}^{2}\norm{K' - K} + L\frac{\sqrt{2}c^{2}\Gamma}{1 - \rho}\rho^{t}c\rho^{t+1}\norm{x_{0}}\norm{K' - K}\nonumber \\ & \leq L\frac{2\sqrt{2}c^{3}\Gamma D_{0}^{2}}{1 - \rho}\rho^{2t+1}\norm{K' - K}. \label{eq:diff nabla G}
    \end{align}
    From this, we conclude that
    \begin{align*}
    & \norm{\E\bracket{\sum_{t = 0}^{\infty}\nabla G_{K}(x_{t+1})(x_{t})^{\top}}  - \E\bracket{\sum_{t = 0}^{\infty}\nabla G_{K}(x_{t+1}')(x_{t}')^{\top}}}_{F}  \\ & \leq \sum_{t = 0}^{\infty}\E\bracket{\norm{\nabla G_{K}(x_{t+1})(x_{t})^{\top} - \nabla G_{K}(x_{t+1}')(x_{t}')^{\top}}_{F}}  \\ & \leq \sum_{t = 0}^{\infty}L \frac{2\sqrt{2}c^{3}\Gamma D_{0}^{2}}{1 - \rho}\rho^{2t+1}\norm{K' - K} \tag{using Eq.~\eqref{eq:diff nabla G}} \\ & \leq \frac{\rho LC_{1}}{(1 + \rho)\sqrt{2}}\norm{K' - K} \\ & \leq \frac{LC_{1}}{2}\norm{K' - K}_{F} \tag{using $ 1/2 \leq \rho \leq 1 $}.
    \end{align*}

    Finally, we establish the bound on $ \norm{E_{K}} $. Recall that $ P_{K_{1}^{\rm lin}}  $ satisfies
    \begin{align*}
        (R + B^{\top}P_{K_{1}^{\rm lin}}B)K_{1}^{\rm lin} = B^{\top}P_{K_{1}^{\rm lin}}A, \quad \text{and} \quad
        (R + B^{\top}P_{K_{1}^{\rm lin}}B)K_{2}^{\rm lin}  = B^{\top}P_{K_{1}^{\rm lin}}C.
    \end{align*}
    From this, we observe that $ E_{K^{\rm lin}} = RK^{\rm lin} - B^{\top}P_{K_{1}^{\rm lin}}(H - BK^{\rm lin}) = 0 $. 
    It follows from the definition of $ E_{K} $ that
    \begin{align*}
    \norm{E_{K}} & = \norm{E_{K} - E_{K^{\rm lin}}} \\ & \leq \norm{R(K - K^{\rm lin})} + \norm{B^{\top}(P_{K_{1}} - P_{K_{1}^{\rm lin}})(H - BK)} + \norm{B^{\top}P_{K_{1}^{\rm lin}}B(K - K^{\rm lin})} 
    \end{align*}
Recall that $ \norm{P_{K_{1}^{\rm lin}}} \leq C_{P} $ by Lemma \ref{lemma: bound P_K}. Also, observe that $ \norm{H - BK} \leq \norm{A - BK_{1}} + \norm{C - BK_{2}} $. Consequently, we have
    \begin{align*}
        \norm{E_{K}}&  \leq 
      (1 + \Gamma^{2}C_{P})\norm{K - K^{\rm lin}} + \Gamma (c_{1} + c_{2})\norm{P_{K_{1}} - P_{K_{1}^{\rm lin}}} \tag{$ \norm{R} \leq 1 $ and $ \norm{B} \leq \Gamma $}\\ & \leq  (1 + \Gamma^{2}C_{P})\norm{K - K^{\rm lin}} + \Gamma (c_{1} + c_{2})\frac{2\Gamma^{3}c^{3}}{(1 - \rho)^{2}}\norm{K - K^{\rm lin}} \tag{using \cite[ Lemma 12]{Qu2020CombiningMA}}\\ & \leq 3(c_{1} + c_{2})\frac{\Gamma^{4}c^{3}}{(1 - \rho)^{2}}\norm{K - K^{\rm lin}}. \tag{substituting $ C_{P} $ from Lemma \ref{lemma: bound P_K}}
    \end{align*}
    
\end{proof}

The last result on $ \Sigma_{K}^{\psi\psi} $ is useful in proving the $ h $-smoothness of the cost function $ \mcal{C}(K) $. Recall $ \Sigma_{K}^{\psi\psi} = \E\bracket{\sum_{t = 0}^{\infty}\psi(x_{t})\psi(x_{t})^{\top}} $.
\begin{Lemma}\label{lemma: bound Sigma_psi}
    Under the same conditions as in Lemma \ref{lemma:trajectory_stability}, we have
    \begin{align*}
    \norm{\Sigma_{K}^{\psi\psi}} \leq \frac{2c^{2}D_{0}^{2}}{1 - \rho}.
    \end{align*}
\end{Lemma}
\begin{proof}
    Note $ \norm{\psi(x)} \leq \ell_{\psi}\norm{x} $ by Lipschitz continuity. Also, by Lemma \ref{lemma:trajectory_stability}, $ \norm{x_{t}} \leq c\rho^{t}\norm{x_{0}} $. Consequently, we have
    \begin{align*}
    \norm{\Sigma_{K}^{\psi\psi}} & \leq \E\bracket{\sum_{t = 0}^{\infty}\norm{\psi(x_{t})}^{2}} \leq \ell_{\psi}^{2}\E\bracket{\sum_{t=0}^{\infty}\norm{x_{t}}^{2}} \leq \frac{\ell_{\psi}^{2}c^{2}}{1 - \rho^{2}}\E\bracket{\norm{x_{0}}^{2}}.
    \end{align*}
    Since $ \ell_{\psi} \leq \sqrt{2} $ and $ \norm{x_{0}} \leq D_{0} $, we conclude that
       $ \norm{\Sigma_{K}^{\psi\psi}} \leq \frac{2c^{2}}{1 - \rho^{2}}D_{0}^{2} \leq \frac{2c^{2}}{1 - \rho}D_{0}^{2}.$

\end{proof}

\subsection{Proof of Theorem \ref{thm: conv of algo}} \label{sec: conv of algo}

In this section, we characterize the gradient estimation, a key step in establishing the convergence rate, in the following lemma.

\begin{Lemma}\label{lemma: estimation C(K)}
    Let $ e_{grad} > 0 $ and $ \nu \in (0, 1) $ be given. Suppose $K \in \Lambda(2\delta/3)$. Under the same conditions as in Theorem \ref{thm: landscape}, when $ r \leq \min\curly{\frac{\delta}{3}, \frac{1}{3h}e_{\text{grad}}} $, $T \geq \frac{1}{1 - \rho_{1}}\log\frac{6\widehat{D}C_{\text{max}}}{e_{\text{grad}}r}$, and
    \begin{align*}
    J \geq \frac{\widehat{D}^{2}}{e_{\text{grad}}^{2}r^{2}}\log \frac{4\widehat{D}}{\nu}\max\curly{36\parenthesis{\mcal{C}(K^{*}) + 2h\delta^{2}}^{2}, 144C_{\max}^{2}},
    \end{align*}
    where $ \widehat{D} = p(n + d) $ and $ C_{\max} = \frac{24(1 + \Gamma)^{2}c_1^{2}D_{0}^{2}}{1 - \rho_1}  $, the following holds with probability at least $ 1 - \nu $,
    \begin{align*}
        \norm{\widehat{\nabla \mcal{C}(K)} - \nabla \mcal{C}(K)}_{F} \leq e_{\text{grad}}.
    \end{align*}

\end{Lemma}

\begin{proof}
    Let $ \text{Ball}(r) $ be the uniform distribution over the ball with radius $ r $ (in Frobenius norm) centered at the origin and $ \text{Sphere}(r) $ be the uniform distribution over the sphere with radius $ r $. 
    Denote $ \mcal{C}_{r}(K) = \E_{U \sim \text{Ball}(r)}\bracket{\mcal{C}(K + U)} $. By \cite[Lemma 1]{10.5555/1070432.1070486}, we have
    \begin{align*}
    \nabla \mcal{C}_{r}(K) = \frac{\widehat{D}}{r^{2}}\E_{U \sim \text{Sphere}(r)}\bracket{\mcal{C}(K + U)U}.
    \end{align*} 
    Define $ \mcal{C}_{j} = \mcal{C}(K + U^{j}) $ with $ U^{j} \sim \text{Sphere}(r)$. Recall $ \widehat{\nabla \mcal{C}(K)} = \frac{1}{J}\sum_{j = 1}^{J}\frac{\widehat{D}}{r^{2}}\widehat{\mcal{C}}_{j}U^{j} $ defined in Algorithm \ref{alg:policy gradient estimation}. We can decompose the gradient estimation error into three terms,
    \begin{align*}
    & \norm{\widehat{\nabla \mcal{C}(K)} - \mcal{C}(K)}_{F} \\ & \leq \underbrace{\norm{\nabla \mcal{C}_{r}(K) - \nabla \mcal{C}(K)}_{F}}_{\coloneqq e_{1}} + \underbrace{\norm{\frac{1}{J}\sum_{j = 1}^{J}\frac{\widehat{D}}{r^{2}}\mcal{C}_{j}U^{j} - \nabla \mcal{C}_{r}(K)}_{F}}_{\coloneqq e_{2}} + \underbrace{\norm{\frac{1}{J}\sum_{j = 1}^{J}\frac{\widehat{D}}{r^{2}}\widehat{\mcal{C}}_{j}U^{j} - \frac{1}{J}\sum_{j = 1}^{J}\frac{\widehat{D}}{r^{2}}\mcal{C}_{j}U^{j}}_{F}}_{\coloneqq e_{3}}. 
    \end{align*}

    In the following, we will show that $ e_{1} \leq e_{\rm grad}/3 $ almost surely, $ e_{2} \leq e_{\rm grad}/3 $ with probability at least $ 1 -\nu/2 $, and $ e_{3} \leq e_{\rm grad}/3 $ with probability at least $ 1 -\nu/2 $.
    Firstly, since $ r \leq \frac{\delta}{3} $, we have $ K + U \in \Lambda(\delta) $, in which the cost function $ \mcal{C}(\cdot) $ is $ h $-smooth. By the definition of $ \nabla \mcal{C}_{r}(K) $, we can deduce with probability one,
    \begin{align}
    e_{1} \leq \E_{U \sim \text{Ball}(r)}\bracket{\norm{\nabla \mcal{C}(K + U) - \nabla \mcal{C}(K)}_{F}} \leq h\E_{U \sim \text{Ball}(r)}\bracket{U} \leq hr \leq \frac{e_{\text{grad}}}{3}, \label{eq:e_1}
    \end{align}
    where we have used that $ r \leq \frac{1}{3h}e_{\text{grad}} $ to reach the last inequality. 

    Next, notice that $ \curly{\frac{\widehat{D}}{r^{2}}\mcal{C}_{j}U^{j}}_{j = 1}^{J} $ are i.i.d. copies with expectation $ \nabla \mcal{C}_{r}(K) $. Since $\norm{U^{j}} \leq r $, the $ h $-smoothness of $ \mcal{C}(\cdot) $ implies with probability one,
    \begin{align*}
    \norm{\frac{\widehat{D}}{r^{2}}\mcal{C}_{j}U^{j}}_{F} & \leq \frac{\widehat{D}}{r}\mcal{C}_{j} \leq \frac{\widehat{D}}{r}\parenthesis{\mcal{C}(K^{*}) + \frac{h}{2}\norm{K + U^{j} - K^{*}}_{F}^{2}}.
    \end{align*} 
    Since $ \norm{K + U^{j}}, \norm{K^{*}} \leq\delta $, we conclude that $ \norm{\frac{\widehat{D}}{r^{2}}\mcal{C}_{j}U^{j}}_{F} \leq  \frac{\widehat{D}}{r}\parenthesis{\mcal{C}(K^{*}) + 2h\delta^{2}} $ almost surely. Furthermore, by the matrix Bernstein inequality \citep[Theorem 12]{gross2011recovering}, we have 
    \begin{align}\label{eq:e_2}
    \P\parenthesis{e_{2} \leq \frac{e_{\text{grad}}}{3} }\geq 1 - 2\widehat{D}\exp\parenthesis{-\frac{(e_{\text{grad}} J/ 3)^{2}}{4J\parenthesis{(\widehat{D}/r)(\mcal{C}(K^{*}) + 2h\delta^{2})}^{2}}} \geq 1 - \nu/2,
    \end{align}
    where we have used the fact that $ J \geq \frac{36 \widehat{D}^{2}}{e_{\text{grad}}^{2}r^{2}}\parenthesis{\mcal{C}(K^{*}) + 2h\delta^{2}}^{2}\log \frac{4\widehat{D}}{\nu} $ to derive the second inequality.

    Finally, to upper bound $e_3$, we further decompose it into two parts. Defining $ \tilde{\mcal{C}}_{j} = \E\bracket{\sum_{t = 0}^{T}\parenthesis{x_{t}^{\top}Qx_{t} + u_{t}^{\top}Ru_{t}}} $ with $ u_{t} = -(K + U^{j})\psi(x_{t}) $, we have the following inequality
    \begin{align*}
    e_{3} \leq \underbrace{\norm{\frac{1}{J}\sum_{j = 1}^{J}\frac{\widehat{D}}{r^{2}}\widehat{\mcal{C}}_{j}U^{j} - \frac{1}{J}\sum_{j = 1}^{J}\frac{\widehat{D}}{r^{2}}\tilde{\mcal{C}}_{j}U^{j}}_{F}}_{\coloneqq e_{4}} + \underbrace{\norm{\frac{1}{J}\sum_{j = 1}^{J}\frac{\widehat{D}}{r^{2}}\tilde{\mcal{C}}_{j}U^{j} - \frac{1}{J}\sum_{j = 1}^{J}\frac{\widehat{D}}{r^{2}}\mcal{C}_{j}U^{j}}_{F}}_{\coloneqq e_{5}}.
    \end{align*}
    To bound $ e_{4} $, note that with probability one,
    \begin{align}
    \abs{\widehat{\mcal{C}}_{j}} & = \abs{\sum_{t = 0}^{T}\bracket{x_{t}^{\top}Qx_{t} + u_{t}^{\top}Ru_{t}}} \nonumber\\ &  \leq \sum_{t = 0}^{\infty}\norm{x_{t}}^{2}\norm{Q} + \norm{\psi(x_{t})}^{2}\norm{(K + U^{j})^{\top}R(K + U^{j})} \nonumber\\ & \leq \parenthesis{\norm{Q} + \ell_{\psi}^{2}\norm{(K + U^{j})^{\top}R(K + U^{j})}}\sum_{t = 0}^{\infty}\norm{x_{t}}^{2} \label{eq:bound hat_C_j},
    \end{align}
    where we have used the $ \ell_{\psi} $-Lipschitz property of $ \psi $. 
    Since $ K + U^{j} \in \Lambda(\delta) $, we have $ \norm{K + U^{j}} \leq 1 + \Gamma $. Also, by Lemma \ref{lemma:trajectory_stability} and Assumption \ref{ass: initialization}, we have $ \norm{x_{t}} \leq c\rho^{t}\norm{x_{0}} \leq c\rho^{t}D_{0} $ almost surely. Consequently, by using the facts $ \norm{Q} \leq 1 $ and $ \ell_{\psi} \leq \sqrt{2} $, Eq.~\eqref{eq:bound hat_C_j} becomes
    \begin{align}
        \abs{\widehat{\mcal{C}}_{j}} \leq (1 + 2 (1 + \Gamma)^{2})\sum_{t = 0}^{\infty}c^{2}\rho^{2t}D_{0}^{2} \leq \frac{3(1 + \Gamma)^{2}c^{2}D_{0}^{2}}{1 - \rho}  \eqqcolon C_{\text{max}}, \label{eq:C_max}
    \end{align}
    As such, since $ \E\bracket{\left. \widehat{\mcal{C}}_{j}U^{j} - \tilde{\mcal{C}}_{j}U^{j} \right\rvert U^{j}} = 0 $,   by the matrix Bernstein inequality, it holds
    \begin{align}\label{eq:e_4}
    \P\parenthesis{e_{4} \leq \frac{e_{\text{grad}}}{6}} = \E\bracket{\P\parenthesis{\left. e_{4} \leq \frac{e_{\text{grad}}}{6}\right\vert \curly{U^{j}}_{j = 1}^{J}}} \geq 1 - 2\widehat{D}\exp\parenthesis{- \frac{(e_{\text{grad}}J/6)^{2}}{4J C_{\max}^{2}}} \geq 1 - \frac{\nu}{2},
    \end{align}
    where we have used the fact that $ J \geq  \frac{144\widehat{D}^{2}C_{\text{max}}^{2}}{e_{\text{grad}}^{2}r^{2}}\log\frac{4\widehat{D}}{\nu}$ in the ultimate inequality. Moreover, since $ \psi(\cdot) $ is $ \ell_{\psi} $-Lipschitz and $\norm{x_{t}} \leq c\rho^{t}\norm{x_{0}} $, we notice that
    \begin{align*}
    \abs{\tilde{\mcal{C}}_{j} - \mcal{C}_{j}} & = \abs{\sum_{t = T+1}^{\infty}\E\bracket{x_{t}^{\top}Qx_{t} + u_{t}^{\top}Ru_{t}}} \\ & \leq \parenthesis{\norm{Q} + \ell_{\psi}^{2}\norm{(K + U^{j})^{\top}R(K + U^{j})}}\sum_{t = T+1}^{\infty}\E\bracket{\norm{x_{t}}^{2}} \\ & \leq \parenthesis{\norm{Q} + \ell_{\psi}^{2}\norm{(K + U^{j})^{\top}R(K + U^{j})}}\rho^{2(T+1)}\sum_{t = 0}^{\infty}c^{2}\rho^{2t}D_{0}^{2} \\ & \leq (1 + 2 (1 + \Gamma)^{2})\rho^{2(T+1)}\sum_{t = 0}^{\infty}c^{2}\rho^{2t}D_{0}^{2}\\ & \leq C_{\text{max}}\rho^{2(T+1)}, 
    \end{align*}
    where the final inequality follows from Eq.~\eqref{eq:C_max}. As such, almost surely we have
    \begin{align}
    e_{5} \leq  \frac{1}{J}\sum_{j = 1}^{J}\frac{\widehat{D}}{r^{2}}\norm{U^{j}}\norm{\title{\mcal{C}}_{j} - \mcal{C}_{j}} \leq \frac{\widehat{D}}{r}C_{\text{max}}\rho^{2(T + 1)} \leq \frac{1}{6}e_{\text{grad}}, \label{eq:e_5}
    \end{align}
    where we have used the fact that $ T \geq \frac{1}{1 - \rho_{1}}\log\frac{6\widehat{D}C_{\text{max}}}{e_{\text{grad}}r} $ to obtain the final inequality. Hence, combining Eq.~\eqref{eq:e_4} and \eqref{eq:e_5}, we conclude $ e_{3} \leq \frac{1}{3}e_{\text{grad}} $ with probability at least $ 1 - \frac{\nu}{2} $, which completes the proof of Lemma \ref{lemma: estimation C(K)} together with Eq.\eqref{eq:e_1} and \eqref{eq:e_2}.
\end{proof}

\section{Conclusions}\label{sec:conclusion}
We consider a nonlinear optimal control problem, characterize the local strong convexity of the cost function, and prove that the globally optimal solution is close to a carefully chosen initialization. Additionally, we design a zeroth-order policy gradient algorithm and establish a convergence result under the proposed policy initialization scheme for the nonlinear control problem. 
We hope these results would shed light on the  efficiency of policy gradient methods for nonlinear optimal control problems when the underlying models are unknown to the decision maker. {Future work includes investigating learning problems for highly nonlinear systems and extending the analysis of quadratic cost functions to more general cost functions.}

\bibliographystyle{plain}
\bibliography{arxiv}

\end{document}